%% file: asyn_sa_arXiv_v3.tex
\documentclass[10pt, letterpaper, twoside]{article} 

\input{macros-asyn-sa-arxiv-v3.tex}

\begin{document} 
\markboth{Yu, Wan, \& Sutton}{Asynchronous SA with RL Applications}

\title{Asynchronous Stochastic Approximation with Applications to Average-Reward Reinforcement Learning%
\thanks{This research was supported by DeepMind, Amii, and the Natural Sciences and Engineering Research Council of Canada (NSERC) under grants RGPIN-2024-04939 and DGECR-2024-00312.} \thanks{A shorter version of this work will appear in \emph{SIAM Journal on Control and Optimization.}}}

\author{{\normalsize Huizhen Yu\textsuperscript{a},
Yi Wan\textsuperscript{a}, 
and Richard S. Sutton\textsuperscript{a,b}}}

\date{} 
\maketitle

\blfootnote{\textsuperscript{a}Department of Computing Science, University of Alberta, Canada}
\blfootnote{\textsuperscript{b}Alberta Machine Intelligence Institute (Amii), Canada}

\blfootnote{{Emails:}\,\texttt{janey.hzyu@gmail.com} (HY,\! corresponding author); \texttt{wan6@ualberta.ca} (YW); \texttt{rsutton@ualberta.ca} (RS)}

\vspace*{-0.8cm}

\noindent{\bf Abstract:} 
This paper investigates the stability and convergence properties of asynchronous stochastic approximation (SA) algorithms, with a focus on extensions relevant to average-reward reinforcement learning. We first extend a stability proof method of Borkar and Meyn to accommodate more general noise conditions than previously considered, thereby yielding broader convergence guarantees for asynchronous SA. To sharpen the convergence analysis, we further examine the shadowing properties of asynchronous SA, building on a dynamical systems approach of Hirsch and Bena\"{i}m. These results provide a theoretical foundation for a class of relative value iteration-based reinforcement learning algorithms---developed and analyzed in a companion paper---for solving average-reward Markov and semi-Markov decision processes.

\medskip 
\noindent{\bf Keywords:}
asynchronous stochastic approximation, stability and convergence, shadowing properties, average-reward reinforcement learning, Markov and semi-Markov decision processes 

%\smallskip 
%\noindent{\bf MSC codes:} 62L20, 90C40, 93E20
% 62L20: stochastic approximation; 90C40: Markov and semi-Markov decision processes; 93E20: optimal stochastic control

\section{Introduction}

Asynchronous stochastic approximation (SA) methods play a central role in reinforcement learning (RL), particularly in model-free algorithms for solving Markov and semi-Markov decision processes (MDPs and SMDPs). These processes are versatile models for sequential decision-making under uncertainty in discrete or continuous time, with applications in robotics, finance, operations research, and beyond \cite{Put94}. This work is motivated by the need to advance RL techniques for average-reward MDPs and SMDPs, where the objective is to optimize sustained long-term performance---a setting in which existing asynchronous SA theory remains inadequate. We focus on establishing general stability and convergence results for asynchronous SA under broad noise conditions arising in average-reward RL. These theoretical developments lay the foundation for a companion paper \cite{YWS25} that analyzes and designs average-reward RL algorithms.

The general asynchronous SA framework we consider is based on the seminal works of Borkar \cite{Bor98,Bor00} and Borkar and Meyn \cite{BoM00}. The algorithms operate in a finite-dimensional space $\R^d$ and, given an initial vector $x_0 \in \R^d$, iteratively compute $x_n \in \R^d$ for $n \geq 1$ using an asynchronous scheme. This asynchrony involves selective updates to individual components at each iteration. Specifically, at the start of iteration $n \geq 0$, a nonempty subset $Y_n \subset \I : = \{ 1, 2, \ldots, d\}$ is randomly selected according to some mechanism. The $i$th component $x_n(i)$ of $x_n$ is then updated as 
\begin{equation} \label{eq-gen-form}
x_{n+1}(i) = x_n(i) + \beta_{n,i} \left(h_i(x_n) + \omega_{n+1} (i) \right), \quad \text{if} \ i \in Y_n,
\end{equation}
or remains unchanged, $x_{n+1}(i) = x_n(i)$, if $i \not\in Y_n$.
This process involves a diminishing random stepsize $\beta_{n,i}$, a Lipschitz continuous function $h : \R^d \to \R^d$ expressed as $h = (h_1, \ldots, h_d)$, and a random noise term $\omega_{n+1} = (\omega_{n+1}(1), \ldots, \omega_{n+1}(d)) \in \R^d$. As will be detailed later, the stepsizes and the choices of the sets $Y_n$ satisfy conditions similar to those introduced by Borkar \cite{Bor98,Bor00}, while the function $h$ meets a stability criterion introduced by Borkar and Meyn \cite{BoM00}. This criterion, expressed via a scaling limit of $h$, imposes conditions on the solutions to the ordinary differential equation (ODE) $\dot{x}= h(x)$ when they are far from the origin, thereby ensuring their stability and other properties.

Building on the above framework, this paper makes two main contributions: first, a stability result for asynchronous SA under more general noise conditions than previously considered; and second, a convergence analysis based on shadowing properties. The former not only addresses key gaps in the literature on average-reward RL but also strengthens the theoretical foundation for extending such methods to broader settings. The latter sharpens existing convergence results in asynchronous SA and, in turn, yields stronger convergence guarantees for RL algorithms that build on this theory. We now review the relevant background and explain how these results build upon and extend prior work.

\subsection{Stability: Prior Work and Our Result}
Stability, referring to the boundedness of the iterates $\{x_n\}$, is crucial for convergence analysis of SA algorithms. Various approaches exist for ensuring stability (see \cite[Chap.\ 4]{Bor23}, \cite{KuY03}). We focus on the Borkar--Meyn stability criterion due to its generality and its suitability for average-reward RL, where the mappings underlying the algorithms generally lack contraction or nonexpansion properties---unlike in discounted or total reward settings (see, e.g., \cite{Tsi94,ABB02,YuB13} and \cite[Chaps.\ 4 and 5]{BeT96}). An alternative approach for stability, also applied in RL, is to relate the original algorithm \eqref{eq-gen-form} to a hypothetical, stable counterpart, such as a scaled iteration~\cite{ABB02} or projected iterates \cite{RBQ20}. However, linking the original iterates to the hypothetical ones is often challenging, involving nonexpansive mappings \cite[Lem.\ 2.1]{ABB02} or requiring bounded differences between original and hypothetical iterates \cite[Assum.\ S5, Sec.\ IV.A]{RBQ20}---conditions difficult to satisfy in average-reward RL. In contrast, the Borkar--Meyn stability criterion does not rely on such conditions.

Our first main result is a stability proof for asynchronous SA in the Borkar--Meyn framework under general conditions on the noise terms $\{\omega_n\}$, which arise naturally in average-reward RL for SMDPs. Specifically, the noise is decomposed into a centered component and a biased component as $\omega_{n+1} = M_{n+1} + \epsilon_{n+1}$, where $\{M_{n+1}\}$ is a martingale difference sequence subject to conditional variance conditions, and $\epsilon_{n+1}$ satisfies $\| \epsilon_{n+1} \| \leq \delta_{n+1} (1 + \| x_n \|)$ with $\delta_{n+1} \to 0$ almost surely (a.s.)\ as $n \to \infty$ (see Assum.~\ref{cond-ns}). In average-reward SMDP applications, the biased term $\epsilon_{n+1}$ arises because the function $h$ depends on expected holding times (i.e., expected durations of state transitions), which are unknown SMDP parameters that RL algorithms can only estimate with increasing accuracy over time. The conditions on $\{M_{n+1}\}$ considered here are also weaker than those typically assumed, accommodating RL settings without a priori lower bounds on expected holding times (see Sec.~\ref{sec-2.3} and \cite{YWS25}).

To the best of our knowledge, these general noise conditions have not previously been treated in the stability analysis of asynchronous SA using the Borkar--Meyn stability criterion (although conditions of this type are standard in the synchronous case \cite{Bor23,KuY03}). Borkar and Meyn provided a stability proof for the \emph{synchronous} setting~\cite{BoM00}, but the stability of \emph{asynchronous} SA was asserted in their theorem \cite[Thm.\ 2.5]{BoM00} without an explicit proof. Moreover, their noise conditions are stronger than ours, requiring $M_{n+1}$ to be generated from $x_n$ and i.i.d.\ random disturbances via a function that is uniformly Lipschitz continuous in $x_n$, as detailed in Rem.~\ref{rmk-cond-noise}. (It should be mentioned, however, that their theorem \cite[Thm.\ 2.5]{BoM00} pertains to a general distributed computing framework with communication delays, which we do not consider.) Within the Borkar--Meyn framework, Bhatnagar~\cite[Thm.\ 1]{Bha11} provided a stability proof for asynchronous SA (with bounded communication delays and $\epsilon_{n+1} = 0$) under the condition that $\| M_{n+1} \| \leq K (1 + \| x_n\|)$ for all $n \geq 0$, for some deterministic constant $K$. This condition is more restrictive than the standard condition on martingale-difference noises, which itself can be too strong for SMDPs (see Sec.~\ref{sec-2.3}).

One prominent class of RL algorithms where such SA results play a central role involves learning methods based on \emph{relative value iteration} (RVI) \cite{Sch71,ScF77,Whi63} for solving average-reward problems. In a seminal paper~\cite{ABB01}, Abounadi, Bertsekas, and Borkar first applied the Borkar--Meyn stability criterion to develop an asynchronous stochastic RVI algorithm, RVI Q-learning, for finite-space MDPs under an average-reward criterion. Their work has since been extended by Wan, Naik, and Sutton \cite{WNS21a,WNS21b} to a broader algorithmic framework and to hierarchical control in average-reward MDPs, as well as by the authors \cite{WYS24} to relax the MDP/SMDP model conditions from unichain to weakly communicating assumptions. These developments in RL all rely on asynchronous SA theory---alongside domain-specific analyses involving the structure of MDPs/SMDPs and the RVI approach---to establish convergence, making the stability of the SA process a foundational issue.

However, the theoretical analyses of RVI Q-learning in the prior studies \cite{ABB01,WNS21a,WNS21b} inherit limitations stemming from gaps in the underlying asynchronous SA theory regarding stability. Although these studies showed that the Borkar--Meyn stability criterion is satisfied by the functions $h$ (cf.\ \eqref{eq-gen-form}) associated with their respective algorithms, for the stability of these asynchronous algorithms, the original study of RVI Q-learning \cite{ABB01} relied on an unproven assertion in \cite[Thm.\ 2.5]{BoM00}, while the later works \cite{WNS21a,WNS21b} argued incorrectly. Bhatnagar's stability result \cite[Thm.~1]{Bha11} partially addresses some of these issues, under the assumption of bounded random rewards per stage. However, it does not apply to one algorithm from \cite{WNS21b} for hierarchical control, which aims to solve an associated average-reward SMDP where the noise conditions required in \cite[Thm.~1]{Bha11} are too restrictive.

Our stability result for asynchronous SA presented in this paper (see Thm.~\ref{thm-1}) not only (a) resolves the open stability question in existing RVI Q-learning algorithms from \cite{ABB01,WNS21a,WNS21b}, but also (b) solidifies the groundwork for further extensions of RVI Q-learning. Both points (a) and (b) were already demonstrated in our recent work \cite{WYS24} mentioned earlier, where we applied the SA results established here. In the companion paper \cite{YWS25}, which builds on this work, we further support point (b) by developing a new, generalized RVI Q-learning algorithm for SMDPs.

As for the means by which we obtain our stability result, the main strategy---beyond several technical subtleties---is to use stopping-time techniques to construct auxiliary processes with desirable properties, enabling the application of Borkar and Meyn's reasoning developed in the context of synchronous SA \cite{BoM00}. This approach allows us to handle various types of noise terms in a unified way, while retaining the structure of their original stability proof.

\subsection{Shadowing: Prior Work and Our Result}
Once stability is ensured, a judicious choice of interpolation for the iterates $\{x_n\}$ allows us to apply existing SA theory and conclude that $\{x_n\}$ converges a.s.\ to a compact, connected subset of $\R^d$ with invariance properties with respect to (w.r.t.)\ the ODE $\dot{x}= h(x)$ (see Thm.~\ref{thm-2}). This limit set is typically larger than the $\omega$-limit set of a single ODE solution \cite{Ben96,BeH96}.

An approach developed by Hirsch and Bena\"{i}m narrows this limit set to an $\omega$-limit set of the ODE by using the concept of `shadowing' from dynamical systems. In the SA context, shadowing refers to the existence of an ODE solution whose forward trajectory asymptotically coincides with an interpolated trajectory of $\{x_n\}$. The notion of shadowing was first introduced and studied by Bowen~\cite{Bow75} for perturbed orbits in dynamical systems. Hirsch~\cite{Hir94} later proved a shadowing theorem, which has since found significant applications in SA by enabling sharper characterizations of the limit set; see Bena\"{i}m and Hirsch~\cite{BeH96} and Bena\"{i}m \cite{Ben96,Ben99}.

These prior studies focused primarily on synchronous SA algorithms. In this work, we build on their results to analyze the shadowing properties of asynchronous SA under additional conditions on the biased noise terms, stepsizes, and asynchrony (see Thm.~\ref{thm-3} and Sec.~\ref{sec-shad}). 
We tailor the convergence theorem to a setting relevant to average-reward RL applications, where the focus is on convergence to equilibrium points, but our proof arguments apply more broadly. This constitutes our second main contribution and, to the best of our knowledge, provides the first shadowing analysis for asynchronous SA in the Borkar--Meyn framework. 

\subsection{Summary of Contributions and Paper Organization}

To summarize the main contributions of this paper: 
\begin{itemize}[leftmargin=0.72cm,labelwidth=!]
\item[(i)] We extend Borkar and Meyn's stability method \cite{BoM00} to address more general noise conditions by employing stopping-time techniques. 
This result (Thm.~\ref{thm-1}), combined with existing SA theory, leads to broader convergence guarantees for asynchronous SA (Thm.\ \ref{thm-2}).
\item[(ii)] We further sharpen the convergence analysis for asynchronous SA algorithms by analyzing their shadowing properties (see Sec.\ \ref{sec-shad}), building on a dynamical systems approach of Hirsch and Bena\"{i}m \cite{Hir94,BeH96,Ben96,Ben99}. This yields enhanced convergence results for asynchronous SA (Thm.~\ref{thm-3}), providing sufficient conditions to ensure convergence to a unique (path-dependent) equilibrium within a compact and connected set of equilibria of the associated ODEs.
\end{itemize}
We apply these results in \cite{YWS25} to develop and analyze a generalized RVI Q-learning algorithm for SMDPs.
An important future research direction is to extend this work to distributed computation frameworks that account for communication delays.

This paper is organized as follows. Section~\ref{sec-sa} introduces the asynchronous SA framework, presents our stability and convergence results, and discusses their applications to average-reward RL. The stability proof and basic convergence proof are given in Sec.~\ref{sec-sa-proofs}, while the shadowing properties are analyzed in Sec.~\ref{sec-shad}. An alternative stability proof under a stronger noise condition from prior works \cite{Bor98,BoM00} is provided in the \hyperref[app-alt-stab]{Appendix}. Section~\ref{sec-conc-rmks} concludes with a brief summary.

\vspace*{0.1cm}
\noindent {\bf Notation:} In this paper, $\|\cdot\|$ denotes either the Euclidean norm $\| \cdot\|_2$ or the infinity norm $\| \cdot\|_\infty$, depending on the context. Additionally, $\ind \{ E\}$ denotes the indicator function of an event $E$, and $\E[\,\cdot\,]$ denotes expectation. For $a, b \in \R$, $a \vee b \= \max \{ a, b\}$ and $a \wedge b \= \min \{ a, b\}$. For convergence of functions, we write $\overset{p}{\to}$ for pointwise convergence and $\overset{u.c.}{\to}$ for uniform convergence on compact subsets of the domain.

\section{Asynchronous SA: Stability and Convergence} 
 \label{sec-sa}

We start with a detailed description of the asynchronous SA framework outlined earlier in \eqref{eq-gen-form}.

\subsection{Algorithmic Framework} \label{sec-2.1}
Let $\alpha_n > 0$, $n \geq 0$, be a given positive sequence of diminishing stepsizes.
Consider an asynchronous SA algorithm of the following form: At iteration $n \geq 0$, choose a subset $Y_n \not= \varnothing$ of $\I$. For $i \not\in Y_n$, let $x_{n+1}(i) = x_n(i)$; and for $i \in Y_n$, with $\nu(n,i) \= \sum_{k=0}^{n-1} \ind \{ i \in Y_k\}$ denoting the cumulative number of updates to the $i$th component prior to iteration $n$, let
\begin{equation} \label{eq-alg0}
    x_{n+1}(i)  = x_n(i)  + \alpha_{\nu(n,i)} \big( h_i (x_n) + M_{n+1}(i) + \epsilon_{n+1}(i) \big), \qquad i \in Y_n.
\end{equation}  
The algorithm is associated with an increasing family $\{\F_n\}_{n \geq 0}$ of $\sigma$-fields, where 
$$\F_n \supset \sigma (x_m, Y_m, M_m, \epsilon_m; m \leq n).$$ 
The following conditions will be assumed throughout. 

\begin{assumption}[Conditions on function $h$] \label{cond-h} \hfill 
\begin{enumerate}[leftmargin=0.75cm,labelwidth=!] 
\item[{\rm (i)}] $h$ is Lipschitz continuous: for some $L_h \geq 0$, 
$\| h(x) - h(y) \| \leq L_h \| x - y\|$ for all $x, y \in \R^{d}$.
\item[{\rm (ii)}] Define $h_c(x) \= h(cx)/c$ for $c \geq 1$. As $c \uparrow \infty$, $h_c \overset{u.c.}{\to} h_\infty : \R^d \to \R^d$.
\item[\rm (iii)] The ODE
$ \dot{x}(t) = h_\infty (x(t)) $
has the origin as its unique globally asymptotically stable equilibrium.
\end{enumerate}
\end{assumption}

\begin{assumption}[Conditions on noise terms $M_n, \epsilon_n$] \label{cond-ns} For all $n \geq 0$, we have:
\begin{enumerate}[leftmargin=0.7cm,labelwidth=!] 
\item[{\rm (i)}] 
$\E [ \| M_{n+1} \| ] < \infty$, $\E [ M_{n+1} \mid \F_n ] = 0$ and $ \E[ \| M_{n+1} \|^2 \mid \F_n ] \leq K_n (1 +\| x_n \|^2)$ a.s., for some $\F_n$-measurable $K_n \geq 0$ with $\sup_n K_n < \infty$ a.s. 
\item[{\rm (ii)}] 
$\| \epsilon_{n+1} \| \leq \delta_{n+1} ( 1 + \| x_n \|)$, where $\delta_{n+1}$ is $\F_{n+1}$-measurable and $\delta_n \overset{n \to \infty}{\to} 0$ a.s. 
\end{enumerate}
\end{assumption} 

\begin{assumption}[Stepsize conditions]  \label{cond-ss} \hfill
\begin{enumerate}[leftmargin=0.7cm,labelwidth=!] 
\item[{\rm (i)}] $\sum_n \alpha_n = \infty$, $\sum_n \alpha_n^2 < \infty$, and 
$\alpha_{n+1} \leq \alpha_n$ for all $n$ sufficiently large.
\item[{\rm (ii)}] For $x \in (0,1)$, 
$\sup_n \frac{\alpha_{[ x n]}}{ \alpha_n} < \infty$,
where $[x n]$ denotes the integral part of $xn$.
\item[{\rm (iii)}] For $x \in (0,1)$, as $n \to \infty$, $\frac{ \sum_{k=0}^{[ y n ]} \alpha_k }{ \sum_{k=0}^{n} \alpha_k} \to 1$ uniformly in $y \in [x, 1]$.
\end{enumerate}
\end{assumption} 

For $x > 0$ and $n \geq 0$, define $N(n,x) \= \min  \left\{ m > n : \sum_{k = n}^m \alpha_k \geq x \right\}$.
\begin{assumption}[Asynchronous update conditions] \label{cond-us} \hfill 
\begin{enumerate}[leftmargin=0.7cm,labelwidth=!]
\item[{\rm (i)}]  For some deterministic $\Delta > 0$, 
$\liminf_{n \to \infty} \nu(n,i)/n  \geq \Delta$ a.s., for all $i \in \I$.
\item[\rm (ii)] For each $x > 0$, the limit $\lim_{n \to \infty} \frac{ \sum_{k = \nu(n,i)}^{\nu(N(n,x), i)} \alpha_k}{ \sum_{k = \nu(n,j)}^{\nu(N(n,x), j)} \alpha_k}$ exists a.s., for all $i, j \in \I$.
\end{enumerate}
\end{assumption} 

\begin{remark} \label{rmk-cond-h} \rm
Assumption~\ref{cond-h} on $h$ is the Borkar--Meyn stability criterion \cite{BoM00}. Note that the functions $h_c$ and $h_\infty$ are also Lipschitz continuous with modulus $L_h$, and $h_\infty(0) = 0$. The required uniform convergence condition $h_c \overset{u.c.}{\to} h_\infty$ is equivalent to $h_c \overset{p}{\to} h_\infty$. For the ODE $\dot{x}(t) = h(x(t))$, this assumption not only implies the boundedness of every solution trajectory $x(t)$ for $t \geq 0$, but also guarantees the existence of at least one equilibrium point, as we will show in Lem.~\ref{lem-Eh}. \myqed
\end{remark}

\begin{remark} \label{rmk-cond-noise}  \rm
Assumption~\ref{cond-ns}(i) relaxes the standard condition on the martingale difference noise terms $\{M_n\}$, which uses a deterministic constant $K$ instead of the $K_n$'s in the conditional variance bounds. We will discuss in Sec.~\ref{sec-2.3} the necessity of each condition in the context of our RL applications.

The noise terms for asynchronous SA in Borkar \cite{Bor98} and Borkar and Meyn \cite{BoM00} satisfy Assum.~\ref{cond-ns} but are more specific: $\epsilon_{n+1} = 0$ and $M_{n+1} = F(x_{n}, \zeta_{n+1})$, where $\{\zeta_n\}$ are exogenous and i.i.d.\ random variables, and $F$ is a function uniformly Lipschitz in its first argument. (Stability of the algorithm is assumed in \cite{Bor98} and asserted in \cite[Thm.\ 2.5]{BoM00} without explicit proof.) In the \hyperref[app-alt-stab]{Appendix}, we provide an alternative stability proof for this specific form of martingale-difference noises, which is slightly simpler than our stability proof under the more general Assum.~\ref{cond-ns}. \myqed
\end{remark}

\begin{remark} \label{rmk-cond-part-async}  \rm
Assumptions~\ref{cond-ss} and~\ref{cond-us} regarding stepsizes and asynchrony are nearly identical to those used in \cite{ABB01} for RVI Q-learning. They accommodate commonly used stepsizes, such as $1/ n$ or $1/ (n \ln n)$, and typical RL scenarios where $\I$ is the space of state-action pairs and the sets $Y_n$ are selected based on Markov chains on $\I$ induced by RL agents (see \cite[Ex.~3]{WYS24} for an example). These conditions, with some minor variations in Assum.~\ref{cond-us}(ii), were originally introduced in a broader asynchronous SA context by Borkar \cite{Bor98,Bor00}, specifically for the stepsize structure $\alpha_{\nu(n,i)}$. Their purpose is to impose partial asynchrony so that the asymptotic behavior of the asynchronous algorithm aligns, on average, with that of a synchronous counterpart, thereby facilitating analysis. (This is evident from the detailed analysis in \citep{Bor98,Bor00}, which shows that the limits in Assum.~\ref{cond-us}(ii) must in fact equal $1$; see also our Lems.~\ref{lem4},~\ref{lem-cvg-2},~\ref{lem-shad2}(i).)

This partial asynchrony is crucial for our average-reward RL applications. While Q-learning achieves stability and convergence in fully asynchronous schemes for both discounted  and certain undiscounted total-reward MDPs \cite{Tsi94,YuB13}, these results do not extend to the average-reward Q-learning algorithms of interest, as their associated mappings are generally neither contractive nor nonexpansive. \myqed
\end{remark}

\subsection{Results} \label{sec-2.2}
We now present our stability and convergence theorems for algorithm (\ref{eq-alg0}). 
The stability theorem, our first main result in this section, parallels Borkar~\cite[Thm.\ 4.1]{Bor23} for synchronous algorithms. Its proof, given in Sec.~\ref{sec-sa-proofs}, extends the Borkar--Meyn stability argument through a stopping-time construction, as noted earlier in the introduction.

\begin{theorem}
\label{thm-1}
For algorithm (\ref{eq-alg0}) under Assums.~\ref{cond-h}--\ref{cond-us}, $\{x_n\}$ is bounded a.s.
\end{theorem}

Combining Thm.\ \ref{thm-1} with established SA theory \cite{Bor98,Bor23} leads to our convergence theorem. This theorem characterizes the asymptotic behavior of both the individual iterates $\{x_n\}$ and a continuous trajectory formed by them, which we introduce first. 

Let $\C ((-\infty, \infty); \R^{d})$ (resp.\ $\C ([0, \infty); \R^{d})$) denote the space of all $\R^{d}$-valued continuous functions on $\R$ (resp.\ $\R_+$), equipped with a metric such that convergence in this space corresponds to uniform convergence on compact intervals. These spaces are complete. By the Arzel\'{a}--Ascoli theorem, a family of functions in either space is relatively compact (i.e., has compact closure) if and only if these functions are equicontinuous and pointwise bounded (see \cite[App.\ A.1]{Bor23} or \cite[Chap.\ 4.2.1]{KuY03}).

Define a linearly interpolated trajectory $\bar x(t)$ from $\{x_n\}$ with aggregated stepsizes 
$$\textstyle{\tl \alpha_n = \sum_{i \in Y_n}  \alpha_{\nu(n, i)}}, \quad n \geq 0,$$ 
as the elapsed times between consecutive iterates. Specifically, for $n \geq 0$, let $\tl t(n) \= \sum_{k=0}^{n -1} \tl \alpha_k$ with $\tl t(0) \= 0$, and define $\bar x(\tl t(n)) \= x_n$ and 
\begin{equation} 
 \bar x(t) \=  x_n +  \tfrac{t - \tl t(n)}{\tl t(n+1) - \tl t(n)} \, ( x_{n+1} - x_n), \ \ \,  t \in [\tl t(n), \tl t(n+1)]. \label{eq-cont-traj2}
\end{equation} 
We refer to the temporal coordinate of $\bar x(t)$ as the `ODE-time.' For the results below, we extend $\bar x(\cdot)$ to a function in $\C ((-\infty, \infty); \R^{d})$ by setting $\bar x(\cdot) \equiv x_0$ on $(-\infty, 0)$.

\begin{theorem}
\label{thm-2}
For algorithm (\ref{eq-alg0}) under Assums.~\ref{cond-h}--\ref{cond-us}, almost surely:
\begin{itemize}[leftmargin=0.65cm,labelwidth=!]
\item[\rm (i)] The sequence $\{x_n\}$ converges to a (possibly sample path-dependent) compact, connected, internally chain transitive,%footnote starts
\footnote{Recall from \cite[Sec.\ 5.1, p.\ 21]{Ben99} and \cite[Chap.\ 2.1, p.\ 16]{Bor23} that a compact invariant set $D$ is called \emph{internally chain transitive} if for any $x, y \in D$ and any $\epsilon, T > 0$, there exists a finite sequence of points in $D$ starting at $x$ and ending in $y$: $x$, $x_0, x_1, \ldots, x_{n-1}, x_n = y$, such that $\| x - x_0 \| < \epsilon$ and, for each $0 \leq i < n$, the ODE trajectory starting from $x_i$ comes within $\epsilon$ distance of $x_{i+1}$ at a time greater than $T$. Internally chain transitive sets and related concepts, introduced by Bowen~\cite{Bow75} and Conley \cite{Con78}, are fundamental in dynamical systems theory. Bena\"{i}m \cite{Ben96} was the first to use them to characterize the asymptotic behavior of SA algorithms. Although our proofs do not explicitly involve these sets, they enter through existing convergence theory and provide the theoretical underpinning.}
%footnote ends
invariant set $D$ of the ODE $\dot{x}(t) = h(x(t))$.
\item[\rm (ii)] With $\bar x(\cdot)$ defined as above, the family $\{\bar x(t + \cdot)\}_{t \in \R}$ is relatively compact in $\C((-\infty, \infty); \R^{d})$, and any limit point of $\bar x(t + \cdot)$ as $t \to \infty$ is a solution of the ODE $\dot{x}(t) = \tfrac{1}{d} h(x(t))$ that remains entirely in $D$.
\end{itemize}
\end{theorem}

The proof of Thm.~\ref{thm-2} will be given in Sec.~\ref{sec-sa-proofs}.
Part (i) of Thm.~\ref{thm-2} parallels \cite[Thm.\ 2.1]{Bor23} for synchronous algorithms, while part (ii) is similar to \cite[Thm.\ 6.1]{Bor23} for asynchronous algorithms, given stability. A key distinction from \cite[Thm.\ 6.1]{Bor23} is our choice of the stepsizes that define the ODE-time in $\bar x(\cdot)$. The stepsize choice in \cite[Thm.\ 6.1]{Bor23} appears less advantageous \emph{under Assums.~\ref{cond-ss} and \ref{cond-us}}, as it not only results in multiple limiting ODEs but also complicates their characterization (see Rem.~\ref{rmk-4} for a detailed discussion).
Here, we obtain a unique limiting ODE (Thm.~\ref{thm-2}(ii)), which, beyond technical convenience, facilitates further analysis of the shadowing properties of $\bar x(\cdot)$, as will be discussed below. 

We now focus on a scenario relevant to RL, where the goal is for the algorithms to converge to the equilibrium set of the ODE $\dot{x}(t) = h(x(t))$, defined as $E_h \= \{ x \in \R^d \mid h(x) = 0 \}$. Before proceeding, we note an implication of Assum.~\ref{cond-h} for $E_h$: 

\begin{lemma} \label{lem-Eh}
If $h$ satisfies Assum.~\ref{cond-h}, then $E_h$ is nonempty, compact, and contained in some compact, connected, globally asymptotically stable%footnote starts
\footnote{Recall that a compact set $D$ is globally asymptotically stable if every solution $x(t)$ of the ODE converges to $D$ as $t \to \infty$, and $D$ is Lyapunov stable (see \cite[Chap.\ 4.2.2]{KuY03}).}
%footnote ends 
set of the ODE $\dot{x}(t) = h(x(t))$.
\end{lemma}

\begin{proof} 
By \cite[Cor.\ 4.1]{Bor23}, there exist $\bar c > 0$ and $T > 0$ such that for any solution $x(\cdot)$ of the ODE $\dot{x}(t) = h(x(t))$, if $\|x(t)\| \geq \bar c$, then $\| x(t + T) \| <  \| x(t) \|/4$. Thus, for any initial condition $x(0)$, there exists a sequence of times $t_n \uparrow \infty$ such that $x(t_n) \in B_{\bar c}: = \{ y \in \R^d \mid \| y \| \leq \bar c\}$. The compact set $B_{\bar c}$ is hence a global weak attractor (by the definition of such an attractor; see \cite[Chap.\ V.1]{BhS02}) and contains at least one equilibrium point by \cite[Thm.~V.3.9]{BhS02}. This shows that $B_{\bar c} \supset E_h \not=\varnothing$. Since $h$ is Lipschitz continuous, $E_h$ is compact. For the desired globally asymptotically stable set containing $E_h$, we can take the first positive prolongation set of $B_{\bar c}$.\footnote{Details: As a compact global weak attractor, $B_{\bar c}$ lies in some compact globally asymptotically stable set, the smallest being its first positive prolongation set $D^+(B_{\bar c})$ \cite[Thm.~V.1.25]{BhS02}. Explicitly, $D^+(B_{\bar c}) = \cup_{x \in   B_{\bar c}} D^+(x)$, where $D^+(x) = \{y \in \R^d \mid \exists \, \{x_n\} \subset \R^d \ \text{and} \  \{t_n\} \subset \R_+ \ \text{s.t.} \  x_n \to x, \, \phi(t_n; x_n) \to y  \ \text{as} \ n \to \infty \}$ with $\phi(t; \bar x)$ denoting the ODE solution with $x(0) = \bar x$. Since each $D^+(x)$, $x \in B_{\bar c}$, is connected \cite[Thm.~II.4.4]{BhS02} and $B_{\bar c}$ is connected, $D^+(B_{\bar c})$ is connected.} 
\end{proof}

As will be discussed in Sec.~\ref{sec-2.3}, for average-reward RL in MDPs/SMDPs without unichain model restrictions, $E_h$ is generally not a singleton but rather consists of infinitely many connected equilibrium points. Corollary~\ref{cor-ql} below applies Thm.~\ref{thm-2} to this setting, using the fact that if $E_h$ is globally asymptotically stable, it must contain all compact invariant sets. 

\begin{cor} \label{cor-ql} 
Let Assums.~\ref{cond-h}--\ref{cond-us} hold, and let $E_h$ be globally asymptotically stable for the ODE $\dot{x}(t) = h(x(t))$. Then the following hold a.s.\ for algorithm \eqref{eq-alg0}:
\begin{itemize}[leftmargin=0.65cm,labelwidth=!]
\item[\rm (i)] The sequence $\{x_n\}$ converges to a compact connected subset of $E_h$.
\item[\rm (ii)] For any $\delta > 0$ and any convergent subsequence $\{x_{n_k}\}$, as $k \to \infty$,
$$\tau_{\delta,k} \= \min \left\{ |s| :  \| \bar x(t_{n_k} + s) - x^* \| > \delta, \ s \in \R \right\} \to \infty,$$ 
where $\bar x(\cdot)$ is the continuous trajectory defined above, 
$t_{n_k} = \tl t(n_k)$ is the ODE-time when $x_{n_k}$ is generated, and $x^* \in E_h$ is the point to which $\{x_{n_k}\}$ converges.
\end{itemize}
\end{cor}

Corollary~\ref{cor-ql}(ii) shows that algorithm \eqref{eq-alg0} spends increasing ODE-time in arbitrarily small neighborhoods around its iterates' limit points in $E_h$, with each visit near a limit point lasting longer and tending to infinity. Thus, when $E_h$ contains non-isolated equilibria, the algorithm's behavior may resemble convergence to a single point, even if it does not truly converge. 

We now analyze this case further%footnote starts
\footnote{Any compact, connected subset of $E_h$ is internally chain transitive by definition, so Cor.~\ref{cor-ql}(i) cannot be further improved based on Thm.~\ref{thm-2}(i).}
%footnote ends 
by presenting sufficient conditions for the asynchronous algorithm~\eqref{eq-alg0} to converge to a unique point in $E_h$, a result also relevant to our RL applications. Our approach builds on Hirsch~\cite{Hir94}, Bena\"{i}m and Hirsch~\cite{BeH96}, and Bena\"{i}m \cite{Ben96,Ben99}. Roughly speaking, the idea is to ensure that $\bar x(\cdot)$ asymptotically tracks a unique solution trajectory of the limiting ODE, a concept known as \emph{shadowing} in the cited literature. One way to achieve this is by ensuring sufficiently rapid tracking of the ODE solutions. In the asynchronous SA setting, we decompose the `tracking error' into two components---one from stochastic noise and one from asynchrony---and introduce conditions to control both. 

We formalize these conditions as follows:

\begin{assumption}[Additional condition on noise term $\epsilon_n$] \label{cond-mns}
There exists a deterministic constant $\mu_\delta < 0$ such that $\limsup_{n \to \infty} \tfrac{\ln(\delta_{n+1})}{\sum_{k=0}^n \alpha_k} \leq \mu_\delta$ a.s., where $\{\delta_{n}\}$ are the random variables involved in Assum.~\ref{cond-ns}(ii) for $\{\epsilon_n\}$.
\end{assumption}

Consider two specific forms of stepsizes: $\alpha_n = \tfrac{1}{A n}$ (class 1) and $\alpha_n = \tfrac{1}{A n \ln n}$ (class 2), where $A > 0$ is a scaling parameter (with $\alpha_n$ set to $\tfrac{1}{A}$ if the denominator is zero). For class-1 stepsizes, we impose an additional condition on asynchony:

\begin{assumption}[Additional asynchrony condition] \label{cond-mus} 
For class-1 stepsizes, the asynchronous update schedules are such that a.s.\ for all $i \in \I$, $\nu(n,i)/n \to p_i$ as $n \to \infty$, for some (sample path-dependent) $p_i \in (0,1]$. Moreover, there exists a deterministic constant $\gamma > 0$ such that $\limsup_{n \to \infty} n^{\gamma} \big| \nu(n,i)/n - p_i \big| < \infty$ a.s.
\end{assumption} 

The following convergence theorem is our second main result in this section; its proof is given in Sec.~\ref{sec-shad}. All conditions can be met in our RL applications (see Sec.~\ref{sec-2.3}). Let $L_h$ be the Lipschitz constant of $h$ under $\| \cdot\|_\infty$---this norm is chosen because, in our RL applications, an effective bound on $L_h$ under this norm can be obtained with minimal or no model knowledge (see Sec.~\ref{sec-2.3}).

\begin{theorem} \label{thm-3}
Consider algorithm~\eqref{eq-alg0} with class-1 stepsizes where $\tfrac{A}{2} > L_h$ or class-2 stepsizes where $A >  L_h$. Let Assums.~\ref{cond-h}-\ref{cond-us},~\ref{cond-mns} with $\mu_\delta < - L_h$ hold, and for class-1 stepsizes, also assume Assum.~\ref{cond-mus} with $\gamma A > L_h$.
For the ODE $\dot{x}(t) = h(x(t))$, suppose that the set $E_h$ is globally asymptotically stable and that every solution $x(t)$ converges to a unique point in $E_h$ as $t \to \infty$. Then the sequence $\{x_n\}$ from algorithm \eqref{eq-alg0} converges a.s.\ to a point in $E_h$ that depends on the sample path. 
\end{theorem}

We now make a few general remarks on the additional conditions and the proof of Thm.~\ref{thm-3}, before discussing the average-reward RL application (Sec.~\ref{sec-2.3}).

\begin{remark}[On the proof] \label{rmk-shad-prf} \rm
Following the general approach of Hirsch and Bena\"{i}m, most arguments in our proof of Thm.~\ref{thm-3} also apply to analyzing shadowing properties of asynchronous SA algorithms in more general settings. In particular, much of the analysis is independent of $E_h$ and its associated conditions---these are only used to specialize the conclusions. Moreover, rather than relying on the specific forms of the stepsizes $\{\alpha_n\}$, the proof primarily depends on the quantity $\ell(\{\alpha_n\}) \= \limsup_{n \to \infty} \tfrac{\ln(\alpha_n)}{\sum_{k=0}^{n} \alpha_k}$, similar to Bena\"{i}m's analysis \cite{Ben96,Ben99} for synchronous SA. Class-1 and class-2 stepsizes have $\ell(\{\alpha_n\}) = -A$ and $-\infty$, respectively, with their specific forms used only to obtain explicit estimates of the `tracking error' due to asynchrony (see Lem.~\ref{lem-shad3-2} in Sec.~\ref{sec-A.1.2}). \myqed
\end{remark}

\begin{remark}[On the additional ODE conditions] \label{rmk-add-ode} \rm
Borkar and Soumyanath~\cite{BoS97} showed that if $h(x) = F(x) - x$, with $F: \R^d \to \R^d$ nonexpansive in the $p$-norm ($p \in (1, +\infty]$) and with nonempty fixed-point set $E_h$, then the solution $x(t)$ of $\dot{x}(t) = h(x(t))$ converges to a point in $E_h$, with distance to $E_h$ nonincreasing. It follows that if $h$ also satisfies the Borkar--Meyn stability criterion, the additional ODE conditions in Thm.~\ref{thm-3} hold. This identifies another class of problems where Thm.~\ref{thm-3} applies, besides average-reward RL. In this setting, Thm.~\ref{thm-3} sharpens the previous result \cite[Thm.\ 3.2]{Bor98} in the absence of communication delays.
\myqed
\end{remark}

\begin{remark}[On the additional stepsize/asynchrony conditions] \label{rmk-add-cond} \rm \hfill \\
(a) For class-2 stepsizes, their rapid decrease eliminates the need for the additional asynchrony condition (Assum.~\ref{cond-mus}), as Assum.~\ref{cond-us} alone suffices to control the tracking error from asynchrony (see the proof of Lem.~\ref{lem-shad3-2}). In the synchronous case, the scaling parameter $A$ for class-2 stepsizes can be chosen freely, provided that Assum.~\ref{cond-mns} holds with $\mu_\delta < - L_h$.
This is consistent with our asynchronous-case proof and aligns with prior shadowing results \cite{Ben96,Ben99} for synchronous SA.\\*[1pt]
(b) For class-1 stepsizes, one way to satisfy the additional asynchrony condition Assum.~\ref{cond-mus} is to select components for updating $x_n$ in a way that eventually follows an irreducible Markov chain on $\I$. The law of the iterated logarithm then ensures that Assum.~\ref{cond-mus} holds for any $0 < \gamma < \tfrac{1}{2}$, allowing the required condition $\gamma A > L_h$ to be met by choosing the stepsize scaling parameter $A > 2 L_h$. \myqed
\end{remark}

\subsection{Applications to Average-Reward RL} \label{sec-2.3}

This subsection briefly overviews the average-reward RL applications studied in our companion paper \cite{YWS25} and recent work \cite{WYS24}. There, $E_h$ corresponds to a compact, connected subset of solutions to the average-reward optimality equation for a weakly communicating%footnote starts
\footnote{I.e., excluding states that are transient under all policies, the remaining states form a single closed communicating class under some policy; see, e.g., \cite[Chap.~8.3]{Put94}.} 
%footnote ends
SMDP or MDP on finite state and action spaces. A fundamental result in MDP theory \cite{ScF78} shows that the full solution set is homeomorphic to an unbounded convex polyhedron, while $E_h$ is homeomorphic to a compact section of this polyhedron---generally a non-singleton set (unlike in the unichain case). 

We apply Cor.~\ref{cor-ql} and Thm.~\ref{thm-3} to analyze RVI Q-learning, treating several existing algorithms in \cite{WYS24} and developing a more general version in \cite{YWS25}. Corollary~\ref{cor-ql} is used to establish convergence to a compact, connected subset of solutions to the optimality equation, while Thm.~\ref{thm-3} is used to ensure convergence to a unique, sample path-dependent solution (see \cite[Secs.~3.3 and~4]{YWS25}).

To apply these SA results, we prove that Assums.~\ref{cond-h} and \ref{cond-ns} on the function $h$ and the noise, along with the ODE and $E_h$ conditions, hold based on the structure of weakly communicating MDPs/SMDPs and the algorithmic design of RVI Q-learning. (This part of the analysis builds on \cite{ABB01,BoS97} and is nontrivial, though it does not invoke the SA results themselves---even though those results shaped the algorithmic design.)

Regarding the noise conditions: for SMDPs, the standard condition on the martingale difference terms $\{M_n\}$ suffices when a lower bound on the expected holding times is known a priori; otherwise, the more general condition Assum.~\ref{cond-ns}(i) is needed. As noted earlier, Assum.~\ref{cond-ns}(ii) on the biased noise terms $\{\epsilon_{n}\}$ is required because the function $h$ depends on the expected holding times, which are estimated from data with increasing accuracy by the RL algorithm. (See \cite[Lem.~4.2]{YWS25} for exactly how Assum.~\ref{cond-ns} is used in our RL context.) For MDPs, which are SMDPs with unit holding times, all $\epsilon_n = 0$, and the standard condition on $\{M_n\}$ suffices.

The remaining assumptions required by our SA results can be satisfied through appropriate choices of stepsizes and asynchronous update schedules. In Rems.~\ref{rmk-cond-part-async} and~\ref{rmk-add-cond}, we already addressed some of these in an application-independent way. Here, we focus on the application-specific aspects. 
For RL in average-reward SMDPs, since the biased noise terms $\{\epsilon_n\}$ arise from estimated expected holding times, the additional noise condition Assum.~\ref{cond-mns} with $\mu_\delta < - L_h$ can be translated into algorithmic requirements on the estimation procedure and ensured accordingly (see \cite[Thm.~3.2 and Lem.~4.8, used in its proof]{YWS25}). In the special case of MDPs, where all $\epsilon_n = 0$, Assum.~\ref{cond-mns} holds automatically with $\mu_\delta = - \infty$.
An effective upper bound on $L_h$ can be obtained with minimal model knowledge---specifically, a lower bound on the minimum expected holding time---and used in place of the exact value of $L_h$ when setting the threshold for the stepsize scaling parameter $A$ (see \cite[Thm.~3.2 and its proof]{YWS25}). 
In summary, all conditions of the preceding SA results can be satisfied in our RL setting through suitable algorithmic choices, with minimal or no model knowledge (see \cite{YWS25} for details).

\section{Proofs for Section~\ref{sec-sa}} \label{sec-sa-proofs}
This section provides the proofs of all results in Sec.~\ref{sec-sa}, except for Thm.~\ref{thm-3}, whose proof appears in Sec.~\ref{sec-shad}. Specifically, we establish the basic stability and convergence theorems---Thms.~\ref{thm-1}, \ref{thm-2}, and Cor.\ \ref{cor-ql}---for the asynchronous SA algorithm~(\ref{eq-alg0}). Assumptions \ref{cond-h}--\ref{cond-us} remain in effect throughout. To avoid repetition, we draw heavily from Borkar \cite{Bor98,Bor23}, focusing on elements essential to the asynchronous algorithm. In this section, we use $\| \cdot\|$ to denote the Euclidean norm $\|\cdot \|_2$, consistent with \cite{Bor23}; norm equivalence allows us to use the stated conditions without modification. 

We use an ODE-based approach from \cite{Bor98}, working with linearly interpolated trajectories formed from the iterates $\{x_n\}$ and connecting these to solutions of \emph{non-autonomous} ODEs of the form $\dot{x}(t) = \lambda(t) g(x(t))$,
where $\lambda(t)$ arises from asynchrony. Depending on the context, $g$ may be $h$, $h_c$, etc. We construct these trajectories differently for the stability and convergence analyses, leading to different $\lambda(\cdot)$ functions. In Sec.~\ref{sec-prel-ana}, we first define and analyze these time-dependent components $\lambda(\cdot)$ and their asymptotic properties, which are crucial for the subsequent proofs (Secs.~\ref{sec-stab}--\ref{sec-cvg}).

\subsection{Preliminary Analysis} \label{sec-prel-ana}

For the stability proof, we use the deterministic stepsize sequence $\{\alpha_n\}$ as the elapsed times between consecutive iterates to define a continuous trajectory $\bar x(\cdot)$. Once stability is established, we switch to a random stepsize sequence in the convergence analysis, both for technical convenience and sharper results. Using random stepsizes in the stability analysis seems non-viable under our noise conditions.

Specifically, for the stability proof, we define a linearly interpolated trajectory $\bar x(t)$ as follows: Let $t(0) \= 0$ and $t(n) \= \sum_{k=0}^{n-1} \alpha_k$ for $n \geq 1$. Let 
\begin{equation} \label{eq-cont-traj1}
  \bar x(t) \=  x_n +  \tfrac{t - t(n)}{t(n+1) - t(n)} \, ( x_{n+1} - x_n), \quad  t \in [t(n), t(n+1)], \ n \geq 0.
\end{equation} 
To define $\lambda(\cdot)$, we first rewrite algorithm (\ref{eq-alg0}) explicitly in terms of $\{\alpha_n\}$ as 
\begin{equation} \label{eq-alg}
    x_{n+1}(i)  = x_n(i) + \alpha_n \, \q(n, i) \left( h_i (x_n) + M_{n+1}(i) + \epsilon_{n+1}(i) \right), \quad i \in \I,
\end{equation}
where $\q(n,i) \= \frac{\alpha_{\nu(n,i)}}{\alpha_n} \ind \{i \in Y_n\}$. We show that $\q(n,i)$ is eventually bounded by a deterministic constant, which we then use to define $\lambda(\cdot)$ and its space $\Upsilon$: 
 
\begin{lemma}  \label{lem1}
For some deterministic constant $C \geq 1$, it holds a.s.\ that for all sufficiently large (sample path-dependent) $n$, $\max_{i \in \I} \q(n,i) \leq C$ and $\sum_{ i \in \I} \q(n,i) \geq 1$. 
\end{lemma}

\begin{proof}
As discussed in \cite[p.\ 842]{Bor98}, Assum.~\ref{cond-ss}(ii), together with $\{\alpha_n\}$ being eventually nonincreasing (Assum.~\ref{cond-ss}(i)), implies that $\sup_n \sup_{y \in [x, 1]} \frac{\alpha_{[ y n]}}{ \alpha_n} < \infty$ for $x \in (0,1)$.
By Assum.~\ref{cond-us}(i), for $n$ sufficiently large, $\min_{i \in \I} \nu(n,i)/n \geq \Delta/2$ a.s.\
Thus, for the finite deterministic constant $C \=  \sup_n \sup_{y \in [\Delta/2, 1]} \frac{\alpha_{[ y n]}}{ \alpha_n} \geq 1$, we have $\max_{i \in \I} \q(n,i) \leq \max_{i \in \I} \frac{\alpha_{\nu(n,i)}}{\alpha_n} \leq C$ for $n$ sufficiently large, a.s.
Since $\{\alpha_n\}$ is eventually nonincreasing and the sets $Y_n \not= \varnothing$, Assum.~\ref{cond-us}(i) also implies that for $n$ sufficiently large, $\sum_{ i \in \I} \q(n,i) =  \sum_{ i \in Y_n} \frac{\alpha_{\nu(n,i)}}{\alpha_n}  \geq 1$ a.s.
\end{proof}

Let $C$ be the constant given in Lem.~\ref{lem1}. Let $\Upsilon$ comprise all Borel-measurable functions that map $t \geq 0$ to a $d \times d$ diagonal matrix with nonnegative diagonal entries bounded by $C$. Two such functions are regarded as the same element if they are equal almost everywhere (a.e.)\ w.r.t.\ the Lebesgue measure.
Similarly to \cite{Bor98}, we equip $\Upsilon$ with the coarsest topology that makes the mappings $\psi_{t,f}: \lambda' \in \Upsilon \mapsto \int_0^t  \lambda'(s) f(s) \, ds$ continuous for all $t > 0$ and $f \in L_2([0, t]; \R^{d})$ (the space of all $\R^{d}$-valued square-integrable functions on $[0,t]$). With this topology, $\Upsilon$ is compact and metrizable by the Banach-Alaoglu theorem and the separability of the Hilbert spaces $L_2([0, t]; \R^{d})$, $t > 0$ (see \cite[Chaps.\ 6.2, A.2]{Bor23}). We regard $\Upsilon$ as a compact metric space with a compatible metric.
Note that any sequence in $\Upsilon$ contains a convergent subsequence.

We now define $\lambda(t)$ as a diagonal matrix-valued, piecewise constant function by  
\begin{equation} \label{eq-def-lambda}
    \lambda(t) \=\text{diag} \big( \, \q(n, 1) \!\wedge\! C, \, \q(n, 2) \!\wedge\! C, \, \ldots, \, \q(n, d) \!\wedge\! C \, \big), \quad t \in [t(n), t(n+1)), \ n \geq 0.
\end{equation} 
For $t \geq 0$, $\lambda(t + \cdot)$ on $[0, \infty)$ is an element of $\Upsilon$. Let $I$ denote the identity matrix. The next lemma characterizes the limit points of $\lambda(t + \cdot)$ as $t \to \infty$. Its proof is similar to that of \cite[Thm.\ 3.2]{Bor98,Bor00}, invoking Assum.~\ref{cond-us}(ii) specifically to apply L'H\^{o}pital's rule and using Lem.\ \ref{lem1} to lower bound the limiting functions. 

\begin{lemma} \label{lem4}
Almost surely, for any sequence $t_n \geq 0$ with $t_n \uparrow \infty$, all limit points of the sequence $\{\lambda(t_n + \cdot)\}_{n \geq 0}$ in $\Upsilon$ have the form
$\lambda^*(t) = \rho(t) I$,
where $\rho(\cdot)$ is a real-valued Borel-measurable function satisfying $\tfrac{1}{d} \leq \rho(t) \leq C$ for all $t \geq 0$.
\end{lemma}

\begin{proof}
Let $\{t^1, t^2, \ldots\}$ be a dense set in $\R_+$. 
Consider a sample path for which Assum.~\ref{cond-us}(i) holds and Assum.~\ref{cond-us}(ii) holds for all $x \in \{t^1, t^2, \ldots\}$. (Note that such sample paths form a set of probability $1$.) By its proof, Lem.~\ref{lem1} holds for such a sample path. 

Given $\{t_n\}$ with $t_n \uparrow \infty$, consider any subsequence $\{\lambda(t_{n_k} + \cdot)\}_{k \geq 0}$ converging to some $\lambda^* \in \Upsilon$. Let $i, j \in \I$. With Assums.~\ref{cond-ss} and~\ref{cond-us} holding, it follows from Lem.~\ref{lem1} and the reasoning given in the proofs of \cite[Thm.\ 3.2]{Bor98,Bor00} that
\begin{equation} \label{eq-prf-lambda}
  \int_0^{t^\ell} \lambda^*_{ii}(s) ds =   \int_0^{t^\ell} \lambda^*_{jj}(s) ds,  \quad \forall \, \ell \geq 1, \ \forall \, i, j \in \I.
\end{equation}   
Since $\{t^\ell\}$ is dense in $\R_+$ and $\lambda^*_{ii}(s) \in [0, C]$, it follows that $f(t) \= \int_{0}^{t} \lambda^*_{ii}(s) ds$ defines the same function $f$ for any $i \in \I$ and hence $\lambda^*_{ii}(s) = \lambda^*_{jj}(s)$ a.e.\ by the Lebesgue differentiation theorem \cite[Thm.\ 7.2.1]{Dud02}.
Since functions in $\Upsilon$ that are identical a.e.\ are treated as the same function, we have $\lambda^*(t) = \rho(t) I$ for some Borel-measurable function $\rho$ with $\rho(t) \in [0,C]$. It remains to show $\rho(t) \geq 1/d$ a.e. By the convergence $\lambda(t_{n_k} + \cdot) \to \lambda^*$ in $\Upsilon$, for all $t, s > 0$,
$$  \int_t^{t+s} \!\rho(y) \, \text{trace}( I) dy = \lim_{k \to \infty} \int_t^{t+s} \!\text{trace}(\lambda(t_{n_k} + y)) dy \geq s,$$
where the inequality follows from Lem.~\ref{lem1} and the definition of $\lambda(\cdot)$. Thus $\int_t^{t+s} \!\rho(y) dy \geq  \tfrac{s}{d}$ for all $t, s > 0$, implying $\rho(t) \geq \tfrac{1}{d}$ a.e.\ by the Lebesgue differentiation theorem \cite[Thm.\ 7.2.1]{Dud02}. 
\end{proof}

\begin{remark} \rm \label{rmk-2}
We make two comments on the preceding proof:\\*[1pt]
(a) The proofs of Borkar \cite[Thm.\ 3.2]{Bor98,Bor00} ingeniously employ L'H\^{o}pital's rule. While these proofs deal with a function $\lambda(\cdot)$ different from ours, the same reasoning is applicable in our case. It shows that under Assums.~\ref{cond-ss} and~\ref{cond-us}, for each $x > 0$, all these limits in Assum.~\ref{cond-us}(ii), $\lim_{n \to \infty} \frac{ \sum_{k = \nu(n,i)}^{\nu(N(n,x), i)} \alpha_k}{ \sum_{k = \nu(n,j)}^{\nu(N(n,x), j)} \alpha_k}$, $i, j \in \I$, must equal to $1$ a.s. This leads to (\ref{eq-prf-lambda}).\\*[2pt]
(b) In the application of the Lebesgue differentiation theorem, alternative measure-theoretical arguments can be employed. Given that $\int_{t}^{t'} \lambda^*_{ii}(s) ds =   \int_{t}^{t'} \lambda^*_{jj}(s) ds$ for all $0 \leq t < t'$, both $\lambda^*_{ii}(s) ds$ and $\lambda^*_{jj}(s) ds$ define the same $\sigma$-finite measure on $\R_+$ according to \cite[Thm.\ 3.2.6]{Dud02}. Consequently, $\lambda^*_{ii}(s) = \lambda^*_{jj}(s)$ a.e.\ by the Radon-Nikodym theorem \cite[Thm.\ 5.5.4]{Dud02}. Given that $\int_t^{t+s} \!\rho(y) dy \geq  \tfrac{s}{d}$ for all $t, s > 0$, by a differentiation theorem for measures \cite[Chap.\ VII, \secmark8]{Doo53}, $\rho(t) \geq \tfrac{1}{d}$ a.e.
\qed
\end{remark}

As mentioned, our convergence proof uses a different setup. Specifically, we work with the trajectory $\bar x(t)$ defined by \eqref{eq-cont-traj2}, which places the iterate $x_n$ at the temporal coordinate $\tl t(n) = \sum_{k=0}^{n -1} \tl \alpha_k$, using aggregated random stepsizes $\tl \alpha_k = \sum_{i \in Y_k}  \alpha_{\nu(k, i)}$. As we show below, this interpolation scheme leads to a simpler limiting behavior of the associated function $\lambda(\cdot)$, denoted by $\tl \lambda(\cdot)$ in this context.

Let us write algorithm (\ref{eq-alg0}) equivalently in terms of $\{\tl \alpha_n\}$ as 
\begin{equation} \label{eq-alg1}
    x_{n+1}(i)  = x_n(i) + \tl \alpha_n \,  \tl \q(n, i) \left( h_i (x_n) + M_{n+1}(i) + \epsilon_{n+1}(i) \right), \qquad i \in \I,
\end{equation}
where 
$\tl \q(n, i) \= \frac{\alpha_{\nu(n, i)}}{\tl \alpha_n} \ind\{i \in Y_n\}$ and thus $\sum_{i \in \I} \tl \q(n,i) = 1$. 
Define $\tilde \lambda(\cdot)$ as a diagonal matrix-valued, piecewise constant trajectory by
\begin{equation} \label{eq-tlambda}
 \tl{\lambda}(t) \=\text{diag} \big( \, \tl \q(n, 1), \, \tl \q(n, 2), \, \ldots, \, \tl \q(n, d)   \big), \qquad t \in [\tl t(n), \tl t(n+1)), \ n \geq 0.
\end{equation}
We view $\tilde \lambda(\cdot)$ as an element in the space $\tilde \Upsilon$ which comprises all Borel-measurable functions that map $t \geq 0$ to a $d \times d$ diagonal matrix with nonnegative diagonal entries summing to $1$. Regarded as a subset of $\Upsilon$ with the relative topology, $\tilde \Upsilon$ is a compact metric space. 

Let us show that as $t \to \infty$, $\tilde \lambda(t + \cdot)$ has a unique limit point in $\tilde \Upsilon$, given by the constant function $\bar \lambda(\cdot) \equiv \tfrac{1}{d} I$. Our proof actually shows more: it also establishes the equivalence of two conditions on asynchrony used in the literature, \cite[Assum.\ 2.4]{ABB01} and the condition introduced earlier in \cite{Bor98,Bor00}. 

Define $\tilde N(n, x) \= \min \left\{ m > n : \sum_{k=n}^{m} \sum_{i \in Y_k} \alpha_{\nu(k, i)} \geq x \right\}$ for $x > 0$.

\begin{lemma} \label{lem-cvg-1}
Given Assums.~\ref{cond-ss} and~\ref{cond-us}(i), Assum.~\ref{cond-us}(ii) is equivalent to that \begin{equation} \label{eq: alt-ns-cond}
   \text{for each $x > 0$},  \ \ \ \textstyle{\lim_{n \to \infty} \frac{ \sum_{k = \nu(n,i)}^{\nu(\tilde N(n,x), i)} \alpha_k}{ \sum_{k = \nu(n,j)}^{\nu(\tilde N(n,x), j)} \alpha_k} = 1} \ \  a.s., \ \ \ \forall \, i, j \in \I.
 \end{equation}
\end{lemma}

\begin{proof} 
First, assume Assum.~\ref{cond-us}(ii).  
Consider a sample path for which Assum.~\ref{cond-us} and Lem.~\ref{lem4} hold. Fix $x > 0$. 
By the definition of $\tilde N(n,x)$ and Assum.~\ref{cond-us}(i), we have
$$\sum_{i \in \I} \sum_{k = \nu(n,i)}^{\nu(\tilde N(n,x), i)} \alpha_k \approx  \sum_{k=n}^{\tilde N(n,x)} \sum_{i \in Y_k} \alpha_{\nu(k, i)} \to x \  \ \text{as $n \to \infty$},$$ 
where we define ``$\approx$'' to denote equality up to an $o(1)$ term.
Thus, \eqref{eq: alt-ns-cond} is equivalent to that 
$$\lim_{n \to \infty}  \sum_{k = \nu(n,i)}^{\nu(\tilde N(n,x), i)} \alpha_k  = \frac{x}{d}, \quad \forall \, i \in \I.$$   
To prove this, it suffices to show that for any increasing sequence $\{n_\ell \}_{\ell \geq 1}$ of natural numbers, there is a subsequence $\{n'_\ell \}_{\ell \geq 1}$ along which $\sum_{k = \nu(n'_\ell,i)}^{\nu(\tilde N(n'_\ell,x), i)} \alpha_k  \overset{\ell \to \infty}{\to} \frac{x}{d}$, $\forall i \in \I$.
To this end, with $t_{n_\ell} \= t(n_\ell)$, consider a convergent subsequence of $\{\lambda(t_{n_\ell} + \cdot)\}_{\ell \geq 1}$ in $\Upsilon$, with limit point $\lambda^*(\cdot) = \rho(\cdot) I$ (Lem.~\ref{lem4}). Denote this convergent subsequence again by $\{\lambda(t_{n_\ell} + \cdot)\}_{\ell \geq 1}$, to simplify notation. Thus, $\lambda(t_{n_\ell} + \cdot) \to \lambda^*$ and we need to prove $\sum_{k = \nu(n_\ell,i)}^{\nu(\tilde N(n_\ell,x), i)} \alpha_k  \overset{\ell \to \infty}{\to} \frac{x}{d}$, $\forall i \in \I$.

Choose $\epsilon \in (0, x)$. Since $\rho(s) \in [1/d, C]$ for all $s \geq 0$ by Lem.~\ref{lem4}, the two equations below define uniquely two constants $\underline{\tau} > 0$ and $\bar \tau > 0$, respectively:
$$\int_0^{\underline{\tau}} \rho(s) ds =  \frac{x - \epsilon}{d}, \qquad \int_0^{\bar{\tau}} \rho(s) ds = \frac{x + \epsilon}{d}.$$
Then, since $\lambda(t_{n_\ell} + \cdot) \to \lambda^*$, we have that for all $i \in \I$, as $\ell \to \infty$, 
$$\int_0^{\underline{\tau}} \lambda_{ii}(t_{n_\ell} + s) ds  \to  \frac{x - \epsilon}{d}, \qquad \int_0^{\bar{\tau}} \lambda_{ii}(t_{n_\ell} + s) ds  \to  \frac{x + \epsilon}{d}.$$
By Lem.~\ref{lem1} and the definition of $\lambda$ [see \eqref{eq-def-lambda}], this implies
\begin{equation} \label{eq-lc1-prf2}
  \underline{c}_\ell(i)   \=   \sum_{k = \nu(n_\ell, i)}^{\nu(N(n_\ell,  \underline{\tau}), i)} \alpha_k \, \to  \, \frac{x - \epsilon}{d}, \qquad
  \bar c_\ell(i)    \=    \sum_{k = \nu(n_\ell, i)}^{\nu(N(n_\ell,  \bar{\tau}),  i)} \alpha_k \, \to \, \frac{x + \epsilon}{d}. 
\end{equation}
Hence 
$$\sum_{k=n_\ell}^{N(n_\ell,\underline{\tau})} \sum_{i \in Y_k} \alpha_{\nu(k, i)}  \approx  \sum_{i \in \I} \underline{c}_\ell(i) \, \to \, x - \epsilon, \qquad \sum_{k=n_\ell}^{N(n_\ell,\bar \tau)} \sum_{i \in Y_k} \alpha_{\nu(k, i)}  \approx \sum_{i \in \I} \bar{c}_\ell(i) \, \to \, x + \epsilon.$$
From these relations and the definition of $\tilde N(n, x)$, it follows that for all $\ell$ sufficiently large,
$N(n_\ell,  \underline{\tau}) <  \tilde N(n_\ell, x) <   N(n_\ell,  \bar{\tau})$ and
consequently,
$$\underline{c}_\ell(i)  \leq   \sum_{k = \nu(n_\ell,  i)}^{\nu(\tilde N(n_\ell,x), i)} \alpha_k \leq \bar c_\ell(i), \qquad \forall \, i \in \I.$$
This together with \eqref{eq-lc1-prf2} and the arbitrariness of $\epsilon$ implies that $\sum_{k = \nu(n_\ell, i)}^{\nu(\tilde N(n_\ell,x), i)} \alpha_k \overset{\ell \to \infty}{\to} \frac{x}{d}$, $\forall i \in \I$, proving that Assum.~\ref{cond-us}(ii) entails the stated condition \eqref{eq: alt-ns-cond}.
 
Conversely, to show that Assum.~\ref{cond-us}(ii) is implied by condition~\eqref{eq: alt-ns-cond}, we argue similarly to the preceding proof but with the roles of $\tilde N(n,\cdot)$ and $N(n,\cdot)$ reversed. Instead of Lem.~\ref{lem4}, we use Lem.~\ref{lem-cvg-2} (an implication of condition~\eqref{eq: alt-ns-cond}, presented below), and we also use the compactness of the space $\Upsilon$. For completeness, we now give the details, although our main results do not use this converse part of the lemma.

Assume \eqref{eq: alt-ns-cond} holds, and consider a sample path for which Assum.~\ref{cond-us}(i) and Lem.~\ref{lem-cvg-2} hold. Fix $x > 0$. To prove the existence of the limits required in Assum.~\ref{cond-us}(ii), it suffices to show that for any increasing sequence $\{m_\ell \}_{\ell \geq 1}$ of natural numbers, there is a subsequence $\{n_\ell \}_{\ell \geq 1}$ along which 
 \begin{equation} \label{eq-lc1-prf3}
  \textstyle{\frac{ \sum_{k = \nu(n_\ell,i)}^{\nu(N(n_\ell,x), i)} \alpha_k}{\sum_{j \in \I} \sum_{k = \nu(n_\ell,j)}^{\nu(N(n_\ell,x), j)} \alpha_k}} \, \to \, \frac{1}{d} \ \  \text{as} \ {\ell \to \infty},  \ \ \ \forall \, i \in \I.
 \end{equation}
To this end, with $t_{n_\ell}$ as defined earlier, take any subsequence $\{n_\ell \}_{\ell \geq 1}$ for which $\{\lambda(t_{n_\ell} + \cdot)\}_{\ell \geq 1}$ converges to some $\lambda^*(\cdot)$ in the compact space $\Upsilon$. Then, as $\ell \to \infty$,
$$\int_0^x \text{trace}\big( \lambda(t_{n_\ell} + s) \big) ds \ \to \ \int_0^x \text{trace}\big( \lambda^*(s) \big) ds = : y,$$
which implies, by the definition of $\lambda(\cdot)$ and Lem.~\ref{lem1} (which does not rely on Assum.~\ref{cond-us}(ii)), that
\begin{equation} \label{eq-lc1-prf4}
   \sum_{k = n_\ell}^{N(n_\ell, x)} \sum_{i \in Y_k} \alpha_{\nu(k,i)} \ \to \ y.
\end{equation}

Let $\epsilon \in (0,y)$. Denote $\tl t_{n_\ell} \= \tl t(n_\ell)$. Since $\tl \lambda (t_{n_\ell} + \cdot)$ converges in $\Upsilon$ to the constant function $\tfrac{1}{d} I$ by Lem.~\ref{lem-cvg-2}, we have that for each $i \in \I$, as $\ell \to \infty$,
$$ \int_0^{y - \epsilon} \tl \lambda_{ii}(\tl t_{n_\ell} + s) ds \to \frac{y - \epsilon}{d}, \qquad \int_0^{y + \epsilon} \tl \lambda_{ii}(\tl t_{n_\ell} + s) ds \to \frac{y + \epsilon}{d}.$$
By the definition of $\tl \lambda(\cdot)$, this implies that 
\begin{equation}\label{eq-lc1-prf5}
 \underline{c}_\ell(i)   \=  \sum_{k = \nu(n_\ell, i)}^{\nu(\tl N(n_\ell, y-\epsilon), i)}  \alpha_{k} \, \to \, \frac{y - \epsilon}{d}, \qquad   \bar c_\ell(i)    \= \sum_{k = \nu(n_\ell, i)}^{\nu(\tl N(n_\ell, y + \epsilon), i)}  \alpha_{k} \, \to \, \frac{y + \epsilon}{d},
\end{equation}
and hence
\begin{equation}\label{eq-lc1-prf6}
   \sum_{k = n_\ell}^{\tl N(n_\ell, y-\epsilon)} \sum_{i \in Y_k} \alpha_{\nu(k,i)} \approx\sum_{i \in \I} \underline{c}_\ell(i) \, \to \, y - \epsilon, \qquad  \sum_{k = n_\ell}^{\tl N(n_\ell, y +\epsilon)} \sum_{i \in Y_k} \alpha_{\nu(k,i)} \approx \sum_{i \in \I} \bar{c}_\ell(i) \, \to \, y + \epsilon.
\end{equation}

From \eqref{eq-lc1-prf4} and \eqref{eq-lc1-prf6} it follows that, for all sufficiently large $\ell$, 
$\tl N(n_\ell, y-\epsilon) < N(n_\ell, x) < \tl N(n_\ell, y+\epsilon)$
and consequently, 
\begin{equation}\label{eq-lc1-prf7}
\underline{c}_\ell(i)  \leq  \sum_{k = \nu(n_\ell,  i)}^{\nu(N(n_\ell,x), i)} \alpha_k \leq \bar c_\ell(i), \quad \forall \, i \in \I.
\end{equation}
Combining \eqref{eq-lc1-prf7} with \eqref{eq-lc1-prf5} and letting $\epsilon \downarrow 0$, we obtain  $\sum_{k = \nu(n_\ell,  i)}^{\nu(N(n_\ell,x), i)} \alpha_k \overset{\ell \to \infty}{\to} \frac{y}{d}$, $\forall \, i \in \I$, and hence \eqref{eq-lc1-prf3} holds. This proves that condition \eqref{eq: alt-ns-cond} entails Assum.~\ref{cond-us}(ii), establishing the equivalence of the two conditions.
 \end{proof}
 
From \eqref{eq: alt-ns-cond} and Assums.~\ref{cond-ss}, \ref{cond-us}(i), Lem.~\ref{lem-cvg-2} follows by \cite[proof of Thm.\ 3.2]{Bor98,Bor00}: 

\begin{lemma}  \label{lem-cvg-2}
As $t \to \infty$, $\tl \lambda(t + \cdot)$ converges a.s.\ in $\tl \Upsilon$ to $\bar \lambda(\cdot) \equiv \tfrac{1}{d} I$.
\end{lemma}

This lemma implies a unique limiting ODE associated with the interpolated trajectory $\bar x(\cdot)$ in \eqref{eq-cont-traj2} (see Sec.~\ref{sec-cvg}). As noted earlier, this uniqueness is not only technically convenient for the convergence analysis but also facilitates further investigation of the shadowing properties of the asynchronous SA algorithm \eqref{eq-alg0} (see Sec.~\ref{sec-shad}). 

\begin{remark}  \label{rmk-4} \rm
A different interpolation scheme is used in \cite[Chap.\ 6.2]{Bor23} and \cite{Bha11} to define the trajectory $\bar x(\cdot)$ and the associated $\tilde \lambda(\cdot)$. In this scheme, $\hat \alpha_n \= \max_{i \in Y_n}  \alpha_{\nu(n, i)}$ represents the elapsed time between the $n$th and the $n+1$th iterates, and the piecewise constant trajectory $\tilde \lambda(\cdot)$ is defined by the ratios $\alpha_{\nu(n, i)} \ind\{ i \in Y_n\} / \hat \alpha_n$, $i \in \I$, during that time interval. As these ratios are bounded by $1$, the resulting $\tilde \lambda(\cdot)$ also lies in a compact metric space. However, it is not clear what the limit points of $\tilde \lambda(t+\cdot)$ are as $t \to \infty$, \emph{under our Assums.~\ref{cond-ss} and~\ref{cond-us} on the stepsizes and asynchrony}. These limit points would be of the form $\rho(t) I$, as in our first scheme, if the reasoning in \cite[Thm.\ 3.2]{Bor98,Bor00} were applicable. That reasoning, however, requires the condition that for each $x > 0$ and 
$\hat N(n, x) \= \min \left\{ m > n : \sum_{k=n}^{m} \max_{i \in Y_k} \alpha_{\nu(k, i)} \geq x \right\}$,
$$ \text{the limit $\lim_{n \to \infty} \tfrac{ \sum_{k = \nu(n,i)}^{\nu(\hat N(n,x), i)} \alpha_k}{ \sum_{k = \nu(n,j)}^{\nu(\hat N(n,x), j)} \alpha_k} $ exists a.s.\ for all $i, j \in \I$,}$$ 
which is identical to our Assum.~\ref{cond-us}(ii) except that $\hat N(n, x)$ replaces $N(n,x)$. 
It is not clear whether this condition is satisfied under our Assums.~\ref{cond-ss} and~\ref{cond-us}. \myqed
\end{remark}

\subsection{Stability Proof} \label{sec-stab}
In this subsection, we prove Thm.~\ref{thm-1} on the boundedness of the iterates $\{x_n\}$. Following the method of Borkar and Meyn \cite{BoM00}, as presented in \cite[Chap.\ 4.2]{Bor23}, we analyze scaled iterates and relate their asymptotic behavior to solutions of certain limiting ODEs involving the function $h_\infty$. Our proof parallels the stability analysis in \cite[Chap.\ 4.2]{Bor23} for synchronous algorithms and is organized into two groups of intermediate results. The first group (Sec.~\ref{sec-stab-prf1}) shows how scaled iterates progressively `track' ODE solutions associated with corresponding scaled functions $h_c$, where we employ stopping arguments to construct auxiliary processes for the asynchronous setting. The second group (Sec.~\ref{sec-stab-prf2}) establishes a stability-related solution property for these ODEs as the scale factor $c$ tends to infinity. With these components in place, the proof concludes in a similar manner to \cite[Chap.\ 4.2]{Bor23}.

\subsubsection{Relating Scaled Iterates to ODE Solutions} \label{sec-stab-prf1}

Consider algorithm \eqref{eq-alg0} in its equivalent form \eqref{eq-alg} and the continuous trajectory $\bar x(t)$ defined in \eqref{eq-cont-traj1}. Following \cite{BoM00} and \cite[Chap.\ 4.2]{Bor23}, we work with a scaled trajectory $\hat x(\cdot)$ derived from $\bar x(\cdot)$ as follows: divide the time axis into intervals of about length $T$, and on each interval, scale $\bar x(\cdot)$ so that the value at the start lies within the unit ball. 

Specifically, let $T > 0$; we will choose a specific value for $T$ later in the proof. Recall each iterate $x_m$ is positioned at time $t(m) = \sum_{k=0}^{m-1} \alpha_k$ with $t(0) = 0$, as defined previously. With $m(0) \=0$ and $T_0 \= 0$, define recursively, for $n \geq 0$,
\begin{equation} \label{eq-def-tm}
   m(n+1) \= \min \{ m : t(m) \geq T_n + T \}, \quad T_{n+1} \= t\big(m(n+1)\big).
\end{equation} 
This divides $[0, \infty)$ into intervals $[T_n, T_{n+1})$, $n \geq 0$, with $T_{n+1} - T_n \to T$ as $n \to \infty$. To simplify expressions, we assume $\sup_n \alpha_n \leq 1$ in this proof, so that each interval is at most $T+1$ in length. 
We then define a piecewise linear function $\hat x(\cdot)$ by scaling $\bar x(t)$ as follows: 
for each $n \geq 0$, with $r(n) \= \| x_{m(n)} \| \vee 1$, 
\begin{equation} \label{eq-hx0}
\hat x(t) \= \bar x(t) / r(n)   \ \ \ \text{for } t \in [T_n, T_{n+1}).
\end{equation}
As $\hat x(\cdot)$ can have `jumps' at times $T_1, T_2, \ldots$, to analyze the behavior of $\hat x(t)$ on the semi-closed interval $[T_n, T_{n+1})$, we introduce, for notational convenience, a `copy' denoted by $\hat x^n(t)$ defined on the closed interval $[T_n, T_{n+1}]$: 
\begin{equation} \label{eq-hx0-cpy}
\hat x^n(t) \= \hat x(t) \ \ \text{for} \  t \in [T_n, T_{n+1}), \qquad \hat x^n(T_{n+1}) \= \hat x(T_{n+1}^-) \= \lim_{t \uparrow T_{n+1}} \hat x(t).
\end{equation}

Let $x^n: [T_n, T_{n+1}] \to \R^{d}$ be the unique solution of the ODE defined by the scaled function $h_{r(n)}$ and the trajectory $\lambda(\cdot)$ given in \eqref{eq-def-lambda}, with initial condition $\hat x(T_n)$:
\begin{equation} \label{eq-ode0}
\dot{x}(t) = \lambda(t) h_{r(n)} (x(t)), \ \ t \in [T_n, T_{n+1}],  \ \ \text{with $x^n(T_n) = \hat x(T_n) = x_{m(n)}/r(n)$}.
\end{equation}
We aim to show that as $n \to \infty$, $\hat x^n(\cdot)$ `tracks' the ODE solution $x^n(\cdot)$.

A key intermediate result is proving $\sup_t \| \hat x(t) \| < \infty$. For \emph{synchronous} SA (under stronger noise conditions than ours), this is shown in \cite[Chap.\ 4.2]{Bor23} in several steps, starting with $\sup_{t} \E [ \| \hat x(t) \|^2] < \infty$, proved through the bound in \cite[Lem.~4.3]{Bor23} that for some constants $\bar K_1, \bar K_2$ independent of $n$, 
\begin{equation} \label{eq-bd-hatx}
 \E \left[ \| \hat x^n(t(k+1)) \|^2 \right]^{\frac{1}{2}} \leq e^{\bar K_1(T+1)} ( 1 + \bar K_2 (T+1) ), \quad m(n) \leq k < m(n+1).
\end{equation} 

In the \emph{asynchronous} case here, we will take a similar approach. However, \eqref{eq-bd-hatx} need not hold for $\{\hat x^n(\cdot)\}$ in our case. 
By \eqref{eq-alg} and the definition of $\hat x(\cdot)$, we have that for $k$ with $m(n) \leq k < m(n+1)$,
\begin{align}  \label{eq-hx}
  \hat x^n(t(k+1)) & = \hat x^n(t(k)) + \alpha_k \Lambda_k h_{r(n)}(\hat x^n(t(k)))  
    + \alpha_k \Lambda_k \hat M_{k+1} + \alpha_k \Lambda_k \hat \epsilon_{k+1},
\end{align}  
where 
$\hat M_{k+1} \= M_{k+1} / r(n)$, $\hat \epsilon_{k+1} \= \epsilon_{k+1}/r(n)$, and 
$$  \Lambda_k \= \text{diag} \big( \q(k, 1), \q(k, 2), \ldots, \q(k, d) \big).$$
Since $r(n) \geq 1$ and $\hat x^n(t(k)) = x_k/r(n)$, by Assum.~\ref{cond-ns}, almost surely,
\begin{align}
\E [ \| \hat M_{k+1} \| ] \!<\! \infty,  \ \ \  & \E [ \hat M_{k+1} \nmid \F_k ] \!= 0,   \  \  \ \E [ \| \hat M_{k+1} \|^2 \nmid \F_k ] \leq K_k ( 1 + \| \hat x^n(t(k)) \|^2),\label{eq-s-noise1} \\ 
& \|\hat \epsilon_{k+1} \| \leq \delta_{k+1} ( 1 + \| \hat x^n(t(k)) \|). \label{eq-s-noise2}
\end{align}
While the scalars $\{ K_k\}_{k \geq 0}$ in \eqref{eq-s-noise1} and the entries of the diagonal matrices $\{\Lambda_k \}_{k \geq 0}$ are bounded a.s.\ by Assum.~\ref{cond-ns}(i) and Lem.~\ref{lem1}, respectively, they need not be bounded by a deterministic constant. While $\delta_{k+1} \to 0$ a.s.\ by Assum.~\ref{cond-ns}(ii), there is no requirement on the conditional variance of $\delta_{k+1}$. These factors prevent us from directly applying the arguments of \cite[Chap.\ 4.2]{Bor23}, which rely on the relation \eqref{eq-bd-hatx}, to prove the desired boundedness of the scaled trajectory $\hat x(\cdot)$.

To work around this issue, we now use stopping techniques to construct `better-behaved' auxiliary processes $\tilde x^n(t)$ on $[T_n, T_{n+1}]$ for $n \geq 0$. Later we will relate them to $\hat x^n(\cdot)$ to establish $\sup_t \| \hat x(t) \| < \infty$.

Let $C$ be the constant given by Lem.~\ref{lem1}, and fix $\bar a > 0$. The following construction applies to each positive integer $N \geq 1$. For notational simplicity, however, we will temporarily suppress the indication of this dependence on $N$ in the constructed processes $\tilde x^n(\cdot)$, $n \geq 0$, until we relate them to the original processes $\hat x^n(\cdot)$.
 
Let $N \geq 1$. For $n \geq 0$, define $k_n \= k_{n,1} \wedge k_{n,2}$, where
\begin{align} 
k_{n,1}  &  \= \min \Big\{  k \, \big| \, K_k > N  \ \text{or} \ \max_{i \in \I} \q(k, i) > C;\, m(n) \leq k < m(n+1) \Big\},  \\ 
 k_{n,2}  & \= \min \left\{  k \, \big| \,  \delta_k  > \bar a, \  m(n)+1 \leq k \leq m(n+1) \right\},   
\end{align}
with $k_{n,1} \= \infty$ and $k_{n,2} \= \infty$ if the sets in their respective defining equations are empty. Then define $\tilde x^n(\cdot)$ on $[T_n, T_{n+1}]$ as follows. 
Let $\tilde x^n(T_n) \= \hat x(T_n)$. For $m(n) \leq k < m(n+1)$, let
\begin{align}  \label{eq-tx}
 \tilde x^n(t(k+1)) & = \tilde x^n(t(k)) + \alpha_k \tilde \Lambda_k h_{r(n)}(\tilde x^n(t(k)))  + \alpha_k \tilde M_{k+1} + \alpha_k \tilde \epsilon_{k+1},
\end{align}   
where $\tilde \Lambda_k \= \ind\{k < k_n\} \Lambda_k$, $\tilde M_{k+1}  \=  \tl \Lambda_k \hat M_{k+1}$, and 
\begin{equation} \label{eq-txa}
  \tilde \epsilon_{k+1} \= \tl \Lambda_k \cdot \ind\{k+1 < k_{n,2}\} \hat \epsilon_{k+1}.
\end{equation}
Finally, on the interval $(t(k), t(k+1))$, let $\tl x^n(t)$ be the linear interpolation between $\tl x^n(t(k))$ and $\tl x^n(t(k+1))$.
As can be seen, in each time interval $[T_n, T_{n+1}]$, 
\begin{align} 
\tilde x^n(t) = \tl x^n (t(k_n)),  \ \ \ &  \text{if} \  t \geq t(k_n) \text{\ (where $t(\infty) \= \infty$});  \notag \\
 \tilde x^n(t)   = \hat x^n (t) \ \  \text{for} \ t \leq t(k_n),  \ \ \    & \text{if} \ k_n = k_{n,1} < k_{n,2};  \label{eq-tx-hx1} \\
   \tilde x^n(t)  = \hat x^n (t) \ \  \text{for} \ t \leq t(k_n-1),  \   \ \ &  \text{if} \ k_n = k_{n,2} \leq k_{n,1}.  \label{eq-tx-hx2}
\end{align}

By definition $k_{n,1}, k_{n,2},$ and $k_n$ are stopping times w.r.t.\ $\{\F_k\}$, so $\ind\{k < k_n\}$ is $\F_k$-measurable and $\ind\{k+1 < k_{n,2}\}$ is $\F_{k+1}$-measurable. Hence $\tl \Lambda_k$ is $\F_k$-measurable, whereas $\tl M_{k+1}$ and $\tl \epsilon_{k+1}$ are $\F_{k+1}$-measurable. As the entries of the diagonal matrices $\Lambda_k$, $m(n) \leq k < k_n$, are all bounded by $C$, the matrix norm of $\tilde \Lambda_k = \ind\{k < k_n\} \Lambda_k$ can be bounded by a deterministic constant $\bar C$:
\begin{equation} \label{eq-Lambda}
   \| \tilde \Lambda_k \| \leq \bar C,  \qquad \forall \,  m(n) \leq k < m(n+1).
\end{equation}   
Moreover, by the construction of $\tilde x^n$ in \eqref{eq-tx}-\eqref{eq-txa} and Assum.~\ref{cond-ns} [cf.\ \eqref{eq-s-noise1}-\eqref{eq-s-noise2}], we have:

\begin{lemma} \label{lem2a}
For $n \geq 0$ and all $k$ with $m(n) \leq k < m(n+1)$, $\E [ \| \tilde M_{k+1} \| ] < \infty$, $\E [ \tilde M_{k+1} \nmid \F_k ] = 0$ a.s., and
\begin{align} 
   \E [ \| \tilde M_{k+1} \|^2 \mid \F_k ]  & \leq \bar C^2 N ( 1 + \| \tilde x^n(t(k)) \|^2) \ \ a.s., \label{eq-prf-noise1} \\
     \| \tilde \epsilon_{k+1} \|  & \leq \bar C (\bar a \wedge \delta_{k+1}) ( 1 + \| \tilde x^n(t(k)) \|) \ \ a.s. \label{eq-prf-noise2}
\end{align}
\end{lemma}

\begin{proof}
Since $\tilde M_{k+1} = \tl \Lambda_k \hat M_{k+1}$, by \eqref{eq-s-noise1} and \eqref{eq-Lambda}, we have
$\E [ \| \tilde M_{k+1} \| ] < \infty$, $\E [ \tilde M_{k+1} \mid \F_k ] = 0$ a.s., and moreover, by the definitions of $\tl \Lambda_k$ and $k_n$,
\begin{equation} \label{eq-l2a-prf1}
\E [ \| \tilde M_{k+1} \|^2 \mid \F_k ] \leq \bar C^2 \ind\{k < k_n\} \, \E [ \| \hat M_{k+1} \|^2 \mid \F_k ].
\end{equation}
If $k < k_n$, $\tilde x^n(t(k)) = \hat x^n(t(k))$ by 
(\ref{eq-tx-hx1})--(\ref{eq-tx-hx2}) and $K_k \leq N$ by the definition of $k_n$. Therefore, by \eqref{eq-s-noise1}, 
\begin{equation} 
 \E [ \| \hat M_{k+1} \|^2 \!\mid \F_k ] \leq K_k ( 1 + \| \hat x^n(t(k)) \|^2) \leq N ( 1 + \| \tl x^n(t(k)) \|^2) \ \  a.s.\ \text{on} \  \{k < k_n\}, \notag
 \end{equation} 
which together with (\ref{eq-l2a-prf1}) proves (\ref{eq-prf-noise1}).

To prove \eqref{eq-prf-noise2}, observe that by the definition of $\tl \epsilon_{k+1}$ [see (\ref{eq-txa})] and the definition of $\tl \Lambda_k$,
$$ \tl \epsilon_{k+1} =  \tl \Lambda_k \cdot \ind\{k < k_n, \, k+1 < k_{n,2}\} \, \hat \epsilon_{k+1}. $$
Then, by \eqref{eq-s-noise2} and \eqref{eq-Lambda}, 
\begin{equation} \label{eq-l2a-prf2}
   \| \tl \epsilon_{k+1} \| \leq \bar C \cdot \ind\{k < k_n, \, k+1 < k_{n,2}\} \, \delta_{k+1} ( 1 + \| \hat x^n(t(k)) \|) \ \ a.s.
\end{equation}   
By the definition of $k_{n,2}$, $\delta_{k+1} \leq \bar a$ if $k + 1 < k_{n,2}$; by \eqref{eq-tx-hx1}--\eqref{eq-tx-hx2}, $\hat x^n(t(k)) = \tilde x^n(t(k))$ if $k < k_n$. Hence, \eqref{eq-l2a-prf2} implies \eqref{eq-prf-noise2}.
\end{proof}

The next lemma applies the proof arguments from \cite[Chap.\ 4.2]{Bor23} to the auxiliary processes $\{\tilde x^n(\cdot)\}$. We will outline the proof, omitting similar details.

\begin{lemma} \label{lem2} The following hold for $\{\tilde x^n(\cdot)\}$:
\begin{enumerate}[leftmargin=0.7cm,labelwidth=!]
\item[\rm (i)] $\sup_{n \geq 0} \sup_{t \in [T_n, \, T_{n+1}]} \E [ \| \tilde x^n(t) \|^2 ] < \infty$.
\item[\rm (ii)] $\tilde \zeta_m \= \sum_{k=0}^{m-1} \alpha_k \tilde M_{k+1}$ converges a.s.\ in $\R^{d}$ as $m \to \infty$.
\item[\rm (iii)] $\sup_{n \geq 0} \sup_{t \in [T_n, \, T_{n+1}]} \| \tilde x^n(t) \|  < \infty$ a.s.
\item[\rm (iv)] $\lim_{n \to \infty} \sup_{t \in [T_n,\, t(k_n) \wedge T_{n+1}]} \left\| \tilde x^n(t) - x^n(t) \right\| = 0$ a.s.
\end{enumerate}
\end{lemma}

\begin{proof}[Proof (outline)]
 By (\ref{eq-tx}) and (\ref{eq-Lambda}), for all $k$ with $m(n) \leq k < m(n+1)$, 
 \begin{align} \label{eq-lem1-prf1}
  \| \tilde x^n(t(k+1)) \| & \leq \| \tilde x^n(t(k))\|  + \alpha_k \bar C \, \| h_{r(n)}(\tilde x^n(t(k))) \|  + \alpha_k \| \tilde M_{k+1}\|  + \alpha_k \| \tilde \epsilon_{k+1} \|.
 \end{align}
Using (\ref{eq-lem1-prf1}), Lem.~\ref{lem2a}, and the bound $\| h_{r(n)}(x) \| \leq \|h(0)\| + L_h \| x\|$ implied by Assum.~\ref{cond-h}, we can follow, step by step, the proof arguments for \cite[Lem.\ 4.3]{Bor23} to derive the following bound, analogous to the bound \eqref{eq-bd-hatx}: For some constants $\bar K_1, \bar K_2$ independent of $n$,  
$$  \E \left[ \| \tilde x^n(t(k+1)) \|^2 \right]^{\frac{1}{2}} \leq e^{\bar K_1(T+1)} ( 1 + \bar K_2 (T+1) ), \quad \forall \, k \ \text{with} \ m(n) \leq k < m(n+1).$$
With this bound, we obtain part (i). This immediately leads to part (ii), similarly to the proof of \cite[Lem.\ 4.4]{Bor23}: By Lem.~\ref{lem2a}, the sequence $\{\tilde \zeta_m\}_{m \geq 0}$ with $\tilde \zeta_0=0$ is a martingale. Moreover, since $\sum_k \alpha_k^2 < \infty$ (Assum.~\ref{cond-ss}(i)), combining part (i) with Lem.~\ref{lem2a}, we obtain  
\begin{align*}
 \E \big[ \| \tl \zeta_m\|^2 \big] \leq \sum_{k=0}^{\infty}  \E \big[  \alpha^2_k \,  \| \tilde M_{k+1} \|^2 \big] & = \sum_{k=0}^{\infty}  \E \Big[ \E \big[  \alpha^2_k \,  \| \tilde M_{k+1} \|^2  \,\big|\, \F_k \big] \Big] \\
 & \leq  \sum_{k=0}^{\infty}  \alpha_k^2 \cdot \bar C^2 N \big( 1 + \E [ \| \tilde x^n(t(k)) \|^2 ] \big) < \infty.
\end{align*} 
Thus $\{\tilde \zeta_m\}_{m \geq 0}$ is a square-integrable martingale satisfying
$\sum_{k=0}^{\infty}  \E [ \, \alpha^2_k \,  \| \tilde M_{k+1} \|^2 \mid \F_k ] < \infty$ a.s.
Part (ii) then follows from the martingale convergence theorem \cite[Prop.\ VII-2-3(c)]{Nev75}.

Finally, part (iii) can be derived from part (ii) just proved, and part (iv) can, in turn, be derived from part (iii), similarly to the proofs of \cite[Lems.\ 4.5 and 2.1]{Bor23}, respectively. In these derivations, in addition to part (ii), we use (\ref{eq-tx}), (\ref{eq-Lambda}), Assum.~\ref{cond-h} on $h$, and (\ref{eq-prf-noise2}) given in Lem.~\ref{lem2a}. For part (iv), we also use $\sum_k \alpha_k^2 < \infty$ and the fact that $\| \tilde \epsilon_k \|  \to 0$ as $k \to \infty$, which is implied by \eqref{eq-prf-noise2}, Assum.~\ref{cond-ns}(ii), and part (iii). 
 \end{proof}

Lemma~\ref{lem2} shows that $\tilde x^n(t)$'s are bounded and asymptotically `track' the ODE solutions $x^n(t)$ (see \eqref{eq-ode0}), albeit over the sub-intervals $[T_n,\, t(k_n) \wedge T_{n+1}]$ of $[T_n,\, T_{n+1}]$. We now use these auxiliary processes to establish the boundedness of the original scaled trajectories $\hat x^n(\cdot)$ and their relationship to the ODE solutions $x^n(\cdot)$.

\begin{lemma} \label{lem3} The following hold for $\{\hat x^n(\cdot)\}$:
\begin{enumerate}[leftmargin=0.7cm,labelwidth=!]
\item[\rm (i)] $\sup_{n \geq 0} \sup_{t \in [T_n, T_{n+1}]} \| \hat x^n(t) \|  < \infty$ a.s.; 
\item[\rm (ii)] $\lim_{n \to \infty} \sup_{t \in [T_n, T_{n+1}]} \left\| \hat x^n(t) - x^n(t) \right\| = 0$ a.s. 
\end{enumerate}
\end{lemma}

\begin{proof}
For each $N \geq 1$, denote the auxiliary processes constructed above by $\tilde x^n_N(\cdot)$ and their associated stopping times by $k_{N,n}$, $n \geq 0$.
By Lem.~\ref{lem2}, a.s., for all $N \geq 1$,
\begin{equation} \label{eq-lem3-prf1}
   \sup_{n \geq 0} \sup_{t \in [T_n, \, T_{n+1}]} \| \tilde x^n_N(t) \|  < \infty, \quad \  \lim_{n \to \infty} \sup_{t \in [T_n, \, t(k_{N,n}) \wedge T_{n+1}]} \left\| \tilde x^n_N(t) - x^n(t) \right\| = 0. 
\end{equation}    
Consider the set $\Omega'$ of sample paths for which \eqref{eq-lem3-prf1} holds for all $N \geq 1$, $\sup_n K_n < \infty$ and, for all sufficiently large $n$, $\delta_n \leq \bar a$ and $\max_{i \in \I} \q(n, i) \leq C$. This set $\Omega'$ has probability $1$ by Lem.~\ref{lem2}, Assum.~\ref{cond-ns}, and Lem.~\ref{lem1}. For each sample path $\omega \in \Omega'$, we have $\sup_n K_n < N(\omega)$ for some positive integer $N(\omega)$ and $k_{N(\omega),n} = \infty$ for all $n$ sufficiently large; consequently, for all $n$ sufficiently large, $t(k_{N(\omega),n}) \wedge T_{n+1} =  T_{n+1}$ and $\hat x^n(\cdot)$ coincides with $\tilde x^n_{N(\omega)}(\cdot)$ by \eqref{eq-tx-hx1}--\eqref{eq-tx-hx2}. Combining this with \eqref{eq-lem3-prf1} yields the conclusions stated in parts (i,\,ii).
\end{proof}

\subsubsection{Stability in Scaling Limits of Corresponding ODEs and Proof Completion} \label{sec-stab-prf2}
 
The ODE solution $x^n(\cdot)$ is determined by $h_{r(n)}$, $\lambda(T_n + \cdot)$, and an initial condition within the unit ball of $\R^{d}$ (see \eqref{eq-ode0}). If $r(n)$ becomes sufficiently large, the initial condition lies on $\unitS \= \{ x \in \R^d \mid \| x \| = 1\}$, and $h_{r(n)}$ approaches the function $h_\infty$ due to Assum.~\ref{cond-h}(ii). By Lem.~\ref{lem4}, a.s., any limit point of $\{ \lambda(T_n + \cdot)\}_{n \geq 0}$ in $\Upsilon$ has the form $\lambda^*(t) = \rho(t) I$ with $\tfrac{1}{d} \leq \rho(t) \leq C$ for all $t \geq 0$.

This leads us to consider \emph{an arbitrary function $\lambda^* \in \Upsilon$ of this form}, the set of which we denote by $\Upsilon^*$, and the associated limiting ODE
\begin{equation} \label{eq-lim-ode}
    \dot{x}(t) = \lambda^*(t) \, h_\infty (x(t)) \quad \text{with} \ \ x(0) \in \unitS.
\end{equation}    
Below, we examine stability properties for such ODEs and their `nearby' ODEs, by generalizing \cite[Lem.\ 4.2, Cor.\ 4.1]{Bor23} from the synchronous context with a single limiting ODE to the asynchronous context with multiple limiting ODEs.

For $x \in \unitS$ and $\lambda^* \in \Upsilon^*$, let $\phi^{\lambda^*}_\infty(t ; x)$ denote the unique solution of (\ref{eq-lim-ode}) with initial condition $\phi^{\lambda^*}_\infty(0; x) = x$.
For $c \geq 1$ and $\lambda' \in \Upsilon$, let $\phi_{c, \lambda'}(t ; x)$ denote the unique solution of $\dot{x}(t) = \lambda'(t) \, h_c (x(t))$ with $\phi_{c, \lambda'}(0; x) = x$. (Existence/uniqueness of these ODE solutions follow from standard Carath\'{e}odory-type results \cite{Wal98}.)

\begin{lemma} \label{lem5}
There exists $T > 0$ such that for all $\lambda^* \in \Upsilon^*$ and initial conditions $x \in \unitS$, $\| \phi^{\lambda^*}_\infty(t; x) \| < 1/8$ for all $t \geq T$.
\end{lemma}

\begin{proof}
For $\lambda^*= \rho(\cdot) I  \in \Upsilon^*$, $\phi^{\lambda^*}_\infty(\cdot; x)$ is related by time scaling to $\phi^o_\infty(\cdot; x)$, the solution of the ODE $\dot{x}(t) = h_\infty (x(t))$ with initial condition $\phi^o_\infty(0; x) = x$; in particular,
$\phi^{\lambda^*}_\infty(t; x) = \phi^o_\infty(\tau(t); x)$, where $\tau(t) = \int_0^t \rho(s) ds$.
Under Assum.~\ref{cond-h}, there exists $T^o > 0$ such that for all $x \in \unitS$, $\| \phi^o_\infty(t; x) \| < 1/8$ for all $t \geq T^o$ \cite[Lem.\ 4.1]{Bor23}.
Since $\rho(t) \geq \tfrac{1}{d}$, this means that for all $t \geq T^o d$, $\|\phi^{\lambda^*}_\infty(t; x) \| = \| \phi^o_\infty(\tau(t); x) \| < 1/8$. 
\end{proof}
 
The next three lemmas extend the stability property of the limiting ODEs (\ref{eq-lim-ode}), as given in the preceding lemma, to `nearby' ODEs within a certain time horizon.

Let $L_h$ be the Lipschitz modulus of $h$ (Assum.~\ref{cond-h}(i)). Consider the set $H$ of all Lipschitz continuous functions $\bar h: \R^d \to \R^d$ with a Lipschitz modulus no greater than $L_h$ and $\|\bar h(0)\| \leq \| h(0)\|$. Endow $H$ with the topology of uniform convergence on compacts and a compatible metric, rendering $H$ a compact metric space by the Arzel\'{a}--Ascoli theorem. The next lemma extends a similar result \cite[Lem.\ 3.1(b)]{Bor98}, which deals with a fixed function $h$ rather than the set $H$, and can be proved using similar arguments:
\begin{lemma} \label{lem6-alt}
Consider the function $\Psi$ that maps each $(\bar \lambda, \bar h, x) \in \Upsilon \times H \times \R^d $ to  the unique solution of the ODE $\dot{x}(t) = \bar \lambda(t) \bar h(x(t))$ on $[0, \infty)$ with initial condition $x(0) = x$. Then $\Psi$ is continuous from $\Upsilon \times H \times \R^d$ into $\C([0, \infty); \R^d)$.
\end{lemma}

\begin{lemma} \label{lem7-alt}
There exist $\bar T > 0$, $\bar c \geq 1$, and a neighborhood $D(\Upsilon^*)$ of $\Upsilon^*$ in $\Upsilon$ such that for all $\lambda' \in D(\Upsilon^*)$ and initial conditions $x \in \unitS$, $\| \phi_{c, \lambda'}(t; x) \| < 1/4$ for all $t \in [\bar T, \bar T+1]$ and $c \geq \bar c$.
\end{lemma}
\begin{proof}
Let $\bar T$ be the time $T$ given by Lem.~\ref{lem5}. 
As the function $\Psi$ is continuous on $\Upsilon \times H \times \R^d$ (Lem.~\ref{lem6-alt}), it is uniformly continuous on the compact set $\Upsilon \times H \times \unitS$. Consequently, there exist a neighborhood $D(h_\infty)$ of $h_\infty$ in $H$ and, for some sufficiently small $\epsilon > 0$,
an $\epsilon$-neighborhood $D(\Upsilon^*)$ of $\Upsilon^*$ in $\Upsilon$ such that
\begin{equation} \label{eq-lem7alt-prf}
\!\! \sup_{t \in [0, \bar T + 1]} \big| \Psi(\lambda', h', x)(t) -  \Psi(\lambda^*_{\lambda'}, h_\infty, x)(t) \big| \leq 1/8, \ \ \  \forall \, x \in \unitS, \, h' \in D(h_\infty), \, \lambda' \in D(\Upsilon^*),
\end{equation}
where $\lambda^*_{\lambda'}$ is any $\lambda^* \in \Upsilon^*$ within distance $\epsilon$ of $\lambda'$.
Since $h_c \in D(h_\infty)$ for all $c$ sufficiently large (Assum.~\ref{cond-h}(ii)), we obtain the desired conclusion by \eqref{eq-lem7alt-prf} and Lem.~\ref{lem5}.
\end{proof}

\begin{remark} \rm
Instead of invoking the continuity of the function $\Psi$, the preceding lemma can also be proven by explicitly computing bounds on $\|\phi^{\lambda^*}_\infty(t; x) - \phi_{c, \lambda'}(t; x)\|$ for $\lambda^* \in \Upsilon^*$, a `nearby' $\lambda' \in \Upsilon$, and $t$ in a given time interval. This proof is analogous to that of \cite[Lem.\ 4.2]{Bor23} and uses Gronwall's inequality as well as the topology on $\Upsilon$. (For details, see Lem.\ 9 and its proof in our earlier report: \href{https://arxiv.org/abs/2312.15091}{\texttt{arXiv:2312.15091}}.)
 \qed
 \end{remark}
 
\emph{Below, let $\bar T > 0$ and $\bar c \geq 1$ be given by Lem.~\ref{lem7-alt}, and use $T \= \bar T + 1/2$ in defining the sequence of times, $T_n, n \geq 0$, for the scaled trajectory $\hat x(\cdot)$ and the solutions $x^n(\cdot)$} introduced previously in Sec.~\ref{sec-stab-prf1} [see \eqref{eq-def-tm}, \eqref{eq-hx0}, and \eqref{eq-ode0}]. 

\begin{lemma} \label{lem7}
Almost surely, there exists a sample path-dependent integer $\bar n \geq 0$ such that for all $n \geq \bar n$, $\lambda'_n \= \lambda(T_n + \cdot) \in \Upsilon$ satisfies that $\| \phi_{c, \lambda'_n}(t; x) \| < 1/4$ for all $t \in [\bar T, \bar T+1]$, $c \geq \bar c$, and 
$x \in \unitS$.
\end{lemma}

\begin{proof}
Consider a sample path for which Lem.~\ref{lem4} holds. Let $G$ be the set of all limit points of $\{\lambda(T_n + \cdot)\}_{n \geq 0}$ in $\Upsilon$. Since $G$ is a closed subset of the compact metric space $\Upsilon$, $G$ is compact. By Lem.~\ref{lem4}, $G \subset \Upsilon^* \subset D(\Upsilon^*)$, where $D(\Upsilon^*)$ is the neighborhood of $\Upsilon^*$ given by Lem.~\ref{lem7-alt}. 
Then, as $\lambda(T_n + \cdot) \overset{n \to \infty}{\to} G$, it follows that for some finite integer $\bar n$, $\lambda'_n = \lambda(T_n + \cdot) \in D(\Upsilon^*)$ for all $n \geq \bar n$. By Lem.~\ref{lem7-alt}, this implies the desired conclusion.
\end{proof}

Finally, noting that $x^n(T_n + \cdot) = \phi_{r(n), \lambda'_n}(\,\cdot\,; \hat x(T_n))$ on $[0, T_{n+1} - T_n]$, we can deduce from Lems.~\ref{lem3} and~\ref{lem7} that the unscaled $\bar x(\cdot)$ has these properties a.s.: 
\begin{itemize}[leftmargin=0.75cm,labelwidth=!]
\item[(i)] $\sup_{t } \|\bar x(t)\| < \infty$ if $\sup_n \|\bar x (T_n) \| < \infty$;
\item[(ii)] $\|\bar x (T_{n+1}) \| \leq \bar c K^*$ if $\|\bar x (T_{n}) \| < \bar c$, where $K^*\= \sup_{t} \| \hat x(t) \|< \infty$ (Lem.~\ref{lem3}(i));
\item[(iii)] for all sufficiently large $n$, $\|\bar x (T_{n+1}) \| \leq \frac{1}{2} \|\bar x (T_{n}) \|$ if $\|\bar x (T_{n}) \| \geq \bar c$.
\end{itemize}
These properties imply the a.s.\ boundedness of $\{x_n\}$, as shown in the proof of \cite[Thm.\ 4.1]{Bor23}.
Details, essentially identical, are provided below for clarity and completeness.

\begin{proof}[Proof of Thm.~\ref{thm-1}] 
The set of sample paths for which Lems.~\ref{lem3} and~\ref{lem7} hold has probability $1$.
Consider any sample path from this set. 
By Lem.~\ref{lem3}(i), $K^* \= \sup_{n \geq 0} \sup_{t \in [T_n, T_{n+1})} \| \hat x(t) \| < \infty$. 
Since $\hat x(t) = \bar x(t)/ r(n)$ on $[T_n, T_{n+1})$, this implies that
\begin{equation} \label{eq-thm1-prf1}
 \sup_{n \geq 0} \| x_n\| = \sup_{n \geq 0} \sup_{t \in [T_n, T_{n+1})} \| \bar x(t) \|  = \sup_{n \geq 0} \sup_{t \in [T_n, T_{n+1})}  r(n) \| \hat x(t) \|  \leq K^* \sup_{n \geq 0} r(n),
\end{equation} 
and
\begin{equation} \label{eq-thm1-prf2}
 \|\bar x(t) \| < \bar c K^* \ \ \  \forall \, t \in [T_n, T_{n+1}], \quad \text{if $r(n) < \bar c$}.
\end{equation}
By Lem.~\ref{lem3}(ii), for all $n$ large enough, $\sup_{t \in [T_n, T_{n+1})} \| \hat x(t) - x^n(t)\| \leq 1/4$. 
This and Lem.~\ref{lem7} together imply that there exists some $\bar n' \geq 0$ such that if $n \geq \bar n'$ and $r(n) \geq \bar c$, then $\|\hat x (T^-_{n + 1})\| < 1/2$, where $\hat x (T^-_{n + 1}) \= \lim_{t \uparrow T_{n + 1}} \hat x(t) = \bar x(T_{n+1})/r(n)$ as defined earlier.
As $r(n) = \| \bar x(T_n)\|$ in this case, we have
\begin{equation} \label{eq-thm1-prf3}
\| \bar x (T_{n + 1})\|  <  \| \bar x (T_{n})\|/2 \quad \text{if  $n \geq \bar n'$ and $r(n) \geq \bar c$}.
 \end{equation}

By \eqref{eq-thm1-prf1}, to prove that $\{x_n\}$ is bounded, it suffices to show $\sup_{n \geq 0} r(n) < \infty$. Assume, for the sake of contradiction, that this is not true. We will use \eqref{eq-thm1-prf3} to derive a contradiction to \eqref{eq-thm1-prf2}. 

Since $r(n) = \| \bar x (T_n) \| \vee 1$, if $\sup_{n \geq 0} r(n) = \infty$, then we can find a subsequence $T_{n_k}, k \geq 0$, with $\bar c \leq r(n_k) = \| \bar x (T_{n_k}) \| \uparrow \infty$. For each $n_k$, let $n'_k \= \max \{ n :   \bar n' \leq n < n_k, \,  r(n) < \bar c \}$ with $n'_k \= \bar n'$ if the set in this definition is empty.  
Then according to (\ref{eq-thm1-prf3}), only two cases are possible: either 
\begin{enumerate}
\item[(i)] $n'_k = \bar n'$ and $r(\bar n') > 2 \, r(\bar n' + 1) > \cdots > 2^{n_k - \bar n'}  r(n_k) \geq 2^{n_k - \bar n'} \bar c$; or
\item[(ii)] $r(n'_k) < \bar c$ and $r(n'_k + 1) > 0.9 \, r(n_k)$. 
\end{enumerate}
Since $r(n_k) \uparrow \infty$, case (i) cannot happen for infinitely many $k$, and we must have case (ii) for all $k$ sufficiently large and with $n'_k \to \infty$ as $k \to \infty$.
Thus we have found infinitely many time intervals $[T_{n'_k}, T_{n'_k+1}]$ during each of which the trajectory $\bar x(\cdot)$ starts from inside the ball of radius $\bar c$ and ends up outside a ball with an increasing radius $0.9 \, r(n_k) \uparrow \infty$. 
But this is impossible by \eqref{eq-thm1-prf2}.

Thus, we obtain $\sup_{n \geq 0} r(n) < \infty$. Then, by \eqref{eq-thm1-prf1}, $\{x_n\}$ must be bounded.
\end{proof}

\subsection{Convergence Given Stability} \label{sec-cvg}
Given the a.s.\ boundedness of the iterates $\{x_n\}$ (Thm.~\ref{thm-1}), we prove the convergence results (Thm.~\ref{thm-2} and Cor.~\ref{cor-ql}) by analyzing the trajectory $\bar x(\cdot)$ defined instead by \eqref{eq-cont-traj2}, where the ODE-time is determined by the aggregated stepsizes $\tilde \alpha_n = \sum_{i \in Y_n} \alpha_{\nu(n,i)}$ for $n \geq 0$, as recalled. For $\tl \lambda(\cdot)$ defined by \eqref{eq-tlambda}, we show similarly to \cite[Lem.\ 2.1]{Bor23} that asymptotically, $\bar x(\cdot)$ `tracks' the solutions of the non-autonomous ODE 
\begin{equation}
   \dot{x}(t) = \tl \lambda(t) h(x(t)).  \label{eq-ode1} 
\end{equation}   
By Lem.~\ref{lem-cvg-2}, the limiting ODE is unique and autonomous:
\begin{equation}
   \dot{x}(t)  = \tfrac{1}{d} h(x(t)), \label{eq-ode2}
\end{equation}
which allows us to apply \cite[Lem.\ 3.1(b)]{Bor98} (cf.\ Lem.~\ref{lem6-alt}) to conclude $\bar x(\cdot)$ `tracks' the solutions of the limiting ODE \eqref{eq-ode2}. We then obtain Thm.~\ref{thm-2}(i) by essentially the same reasoning as in \cite[Thm.\ 2.1]{Bor23} for synchronous SA. This analysis also yields Thm.~\ref{thm-2}(ii) by standard arguments, and Cor.~\ref{cor-ql} follows directly from Thm.~\ref{thm-2}. 
We now provide the proof details.

Let the continuous trajectory $\bar x(t)$ be defined according to \eqref{eq-cont-traj2}:
 \begin{equation} 
 \bar x(t) \=  x_n +  \tfrac{t - \tl t(n)}{\tl t(n+1) - \tl t(n)} \, ( x_{n+1} - x_n), \quad  t \in [\tl t(n), \tl t(n+1)],   \ n \geq 0, \notag
\end{equation} 
where $\tl t(n)= \sum_{k=0}^{n -1} \tl \alpha_k$ and $\tl \alpha_n = \sum_{i \in Y_n}  \alpha_{\nu(n, i)}$.
Consider algorithm \eqref{eq-alg0} in its equivalent form \eqref{eq-alg1}; that is, in vector notation, 
\begin{equation} \label{eq-alg1a0}
    x_{n+1}  = x_n + \tl \alpha_n  \tl \Lambda_n \left( h (x_n) + M_{n+1} + \epsilon_{n+1} \right),
\end{equation}
where $\tl \Lambda_n \=  \text{diag} \big( \tl \q(n, 1), \tl \q(n, 2), \ldots, \tl \q(n, d) \big)$, with diagonal entries $\tl \q(n, i) = \frac{\alpha_{\nu(n, i)}}{\tl \alpha_n} \ind\{i \in Y_n\} \in [0,1]$, as defined previously. 
Note that by Assums.~\ref{cond-ss}(i) and~\ref{cond-us}(i), $\tl \alpha_n = \sum_{i \in Y_n}  \alpha_{\nu(n, i)}$ satisfies 
\begin{equation} \label{eq-stepsize}
 \textstyle{\sum_{n \geq 0} \tl \alpha_n = \infty, \quad \ \ \sum_{n \geq 0} {\tl \alpha}^2_n < \infty}, \ \ \text{a.s.}
\end{equation} 

\begin{lemma} \label{lem-cvg-3}
The sequence $\zeta_n \= \sum_{k=0}^{n-1} \tl \alpha_k \tl \Lambda_k M_{k+1}$, $n \geq 1$, converges a.s.\ in $\R^{d}$.
\end{lemma}

\begin{proof}
For integers $N \geq 1$, define stopping times $\tau_N$ and auxiliary variables $M^{(N)}_{k}$ as follows:
$$\tau_N \= \min \{ k \geq 0: \| x_k \| > N \ \text{or} \ K_k > N \}, \qquad M^{(N)}_{k+1} \= \ind \{ k < \tau_N \} M_{k+1},  \ \ \ k \geq 0.$$
By Assum.~\ref{cond-ns}(i), for each $N$, $\{M^{(N)}_k \}_{k \geq 1}$ is a martingale-difference sequence with $\E[ \| M^{(N)}_{k+1} \|^2 \mid \F_k ] \leq N (1 + N^2 )$. Then the sequence $\{ \zeta^{(N)}_n\}_{n \geq 0}$ given by $\zeta^{(N)}_n \= \sum_{k=0}^{n-1} \tl \alpha_k \tl \Lambda_k M^{(N)}_{k+1}$ with $\zeta^{(N)}_0 \= 0$ is a square-integrable martingale (since the diagonal matrix $\tl \alpha_k \tl \Lambda_k$ has diagonal entries $\alpha_{\nu(k,i)} \ind \{ i \in Y_k\}, i \in \I$, all bounded by the finite constant $\sup_n \alpha_n$). Furthermore, since almost surely, 
$$\sum_{n=0}^\infty \E \left[ \| \zeta^{(N)}_{n+1} - \zeta^{(N)}_n \|^2 \,\big|\, \F_n \right] \leq \sum_{n=0}^\infty \tl \alpha^2_n \| \tl \Lambda_n \|^2 \, \E \left[ \| M^{(N)}_{n+1} \|^2 \,\big|\, \F_n \right] \leq \sum_{n=0}^\infty \tl \alpha_n^2 \| \tl \Lambda_n \|^2 N (1 + N^2) < \infty$$
(where the last inequality follows from (\ref{eq-stepsize}) and the fact that the entries of $\tl \Lambda_n$ lie in $[0,1]$), we have that $\{ \zeta^{(N)}_n\}_{n \geq 0}$ converges a.s.\ in $\R^{d}$ by \cite[Prop.\ VII-2-3(c)]{Nev75}.
As $\{x_n\}$ is bounded a.s.\ by Thm.~\ref{thm-1} and $\sup_n K_n < \infty$ a.s.\ by Assum.~\ref{cond-ns}(i), the definitions of $\tau_N$ and $\{M^{(N)}_k\}$ imply that almost surely, $\{\zeta_n\}_{n \geq 1}$ coincides with $\{\zeta^{(N)}_n\}_{n \geq 1}$ for some sample path-dependent value of $N$, leading to the a.s.\ convergence of $\{\zeta_n\}_{n \geq 1}$ in $\R^{d}$.
 \end{proof}

The next step in the proof involves using Lem.~\ref{lem-cvg-3} and Thm.~\ref{thm-1} to show that the trajectory $\bar x(\cdot)$ asymptotically `tracks' the solutions of two ODEs. The first is the non-autonomous ODE \eqref{eq-ode1}, $\dot{x}(t)  = \tl \lambda(t) h (x(t))$, defined by the random trajectory $\tl \lambda(\cdot) \in \tilde \Upsilon$ (see \eqref{eq-tlambda}). The second is the autonomous limiting ODE \eqref{eq-ode2} $\dot{x}(t) = \tfrac{1}{d} h(x(t))$, obtained using Lem.~\ref{lem-cvg-2}.
 
Let $T > 0$. For $s \geq 0$, let ${\tl x}^s(\cdot)$ and $x^s(\cdot)$ be the unique solutions of (\ref{eq-ode1}) and (\ref{eq-ode2}), respectively, on the time interval $[s, s+T]$ with initial conditions ${\tl x}^s(s) = x^s(s) = \bar x(s)$. For $s \geq T$, let ${\tl x}_s(\cdot)$ and $x_s(\cdot)$ be the unique solutions of (\ref{eq-ode1}) and (\ref{eq-ode2}), respectively, on the time interval $[s-T, s]$ with terminal conditions ${\tl x}_s(s) = x_s(s) = \bar x(s)$.

\begin{lemma} \label{lem-cvg-4}
For any $T > 0$, almost surely,
\begin{align} 
 \lim_{s \to \infty} \sup_{t \in [s, s+T]} \| \bar x(t) - \tl x^s(t)\|  & = 0,   & 
   \lim_{s \to \infty} \sup_{t \in [s-T, s]} \| \bar x(t) - \tl x_s(t)\|  & = 0, \label{eq-lc4-1b} \\ 
   \lim_{s \to \infty} \sup_{t \in [s, s+T]} \| \bar x(t) - x^s(t)\| & = 0, \ & 
   \lim_{s \to \infty} \sup_{t \in [s-T, s]} \| \bar x(t) - x_s(t)\|  & = 0. \label{eq-lc4-2b}
\end{align}
\end{lemma}
  
\begin{proof}
Consider a sample path for which Thm.~\ref{thm-1}, Lems.~\ref{lem-cvg-2} and \ref{lem-cvg-3}, and Assums.~\ref{cond-ns} and~\ref{cond-us} hold. 
To prove (\ref{eq-lc4-1b}), we work with (\ref{eq-alg1a0}) and observe the following: 
\begin{enumerate}
\item[(i)] $\{x_n\}$ is bounded by Thm.~\ref{thm-1}; 
\item[(ii)] $\sum_n \tl \alpha_n = \infty$, $\sum_n {\tl \alpha}^2_n < \infty$ by (\ref{eq-stepsize});
\item[(iii)] $\|\tl \Lambda_n \|, n \geq 0$, and $\tl \lambda(t), t \geq 0$ are bounded by deterministic constants by definition; 
\item[(iv)] $h$ is Lipschitz continuous by Assum.~\ref{cond-h}(i); 
\item[(v)] as $n \to \infty$,  $\sup_{m \geq 0} \left\| \sum_{k=n}^{n+m} \tl \alpha_k \tl \Lambda_k M_{k+1} \right\| \to 0$ by Lem.~\ref{lem-cvg-3}; and $\epsilon_n \to 0$ by Assum.~\ref{cond-ns}(ii) and Thm.~\ref{thm-1}.
\end{enumerate}
Using the above observations, we can essentially replicate the proof of \cite[Lem.~2.1]{Bor23} step by step, with some minor variations, to obtain (\ref{eq-lc4-1b}).

To prove the two equalities in \eqref{eq-lc4-2b}, we first prove the corresponding relations when $\bar x$ is replaced by $\tl x^s$ and $\tl x_s$, respectively. This proof involves Thm.~\ref{thm-1}, Lem.~\ref{lem-cvg-2}, and an application of Borkar \cite[Lem.\ 3.1(b)]{Bor98} (cf.\ Lem.~\ref{lem6-alt}), which deals with solutions of ODEs of the form \eqref{eq-ode1} and their simultaneous continuity in both the $\tl \lambda$ function and the initial condition. Combining this result with \eqref{eq-lc4-1b} then leads to \eqref{eq-lc4-2b}. We now give the details.

Let $\C([0, T]; \R^{d})$ denote the space of all $\R^{d}$-valued continuous functions $f$ on $[0, T]$ with the sup-norm $\| f\| \= \sup_{t \in [0, T]} \| f(t)\|$. Let $\Psi_1$ (respectively, $\Psi_2$) denote the mapping that maps each $(\lambda', x^o) \in \tl \Upsilon \times \R^{d}$ to the unique solution of the ODE $\dot{x}(t) = \lambda'(t) h (x(t)), t \in [0,T]$, with the initial condition $x(0) = x^o$ (respectively, the terminal condition $x(T) = x^o$). Since $h$ is Lipschitz continuous, by \cite[Lem.\ 3.1(b)]{Bor98} (cf.\ Lem.~\ref{lem6-alt}), $\Psi_1$ and $\Psi_2$ are continuous mappings from $\tl \Upsilon \times \R^{d}$ into the space $\C([0, T]; \R^{d})$. Therefore, $\Psi_1$ and $\Psi_2$ are uniformly continuous on any compact subset of $\tl \Upsilon \times \R^{d}$, in particular, on the compact set $\tl \Upsilon \times \overline{\{\bar x(t) : t \geq 0\}}$, where $\overline{\{\bar x(t) : t \geq 0\}}$ denotes the closure of the set $\{\bar x(t) : t \geq 0\}$ and is compact by Thm.~\ref{thm-1}.
Consequently, since $\tl \lambda (t + \cdot) \to \bar \lambda(\cdot) \equiv \tfrac{1}{d} I$ as $t \to \infty$ (Lem.~\ref{lem-cvg-2}) and the initial (respectively, terminal) conditions $\tl x^s(s) = x^s(s)$ (respectively, $\tl x_s(s)=x_s(s)$) all lie in $\{\bar x(t) : t \geq 0\}$, we obtain
$$\lim_{s \to \infty} \sup_{t \in [s, s+T]} \| \tl x^s(t) - x^s(t) \|  = 0 \quad  \text{and} \ \ \  
 \lim_{s \to \infty} \sup_{t \in [s-T, s]} \| \tl x_s(t) - x_s(t)\|  = 0.$$
Together with (\ref{eq-lc4-1b}) proved earlier, this implies (\ref{eq-lc4-2b}).
\end{proof}

Finally, we are ready to prove Thm.~\ref{thm-2} and Cor.~\ref{cor-ql}. Recall that in these results, $\bar x(\cdot)$ is extended to a function in $\C \big((-\infty, \infty); \R^{d} \big)$ by setting $\bar x(\cdot) \equiv x_0$ on $(- \infty, 0)$.

\begin{proof}[Proof of Thm.~\ref{thm-2}]
For part (i), using the a.s.\ boundedness of $\{x_n\}$ given by Thm.~\ref{thm-1} and (\ref{eq-lc4-2b}) given by Lem.~\ref{lem-cvg-4}, the same proof of \cite[Thm.~2.1]{Bor23} goes through here and establishes that $\{x_n\}$ converges a.s.\ to a, possibly sample path-dependent, compact connected internally chain transitive invariant set of the ODE $\dot{x}(t) = \tfrac{1}{d} h (x(t))$. The solutions of this ODE are simply the solutions of the ODE $\dot{x}(t) = h (x(t))$ by a constant time scaling, so the two ODEs have identical compact connected internally chain transitive invariant sets. The desired conclusion then follows. 

For part (ii), consider a sample path for which \eqref{eq-stepsize} and Thms.~\ref{thm-1} and~\ref{thm-2}(i) hold, and Lem.~\ref{lem-cvg-4} holds for all $T = 1 , 2, \ldots$. By Thm.~\ref{thm-1}, $\{\bar x(t + \cdot)\}_{t \in \R}$ is uniformly bounded. Since $h$ is Lipschitz continuous, applying Gronwall's inequality \cite[Lem.\ B.1]{Bor23} shows that given a bounded set of initial conditions $x(0)$, the solutions of the ODE $\dot{x}(t) = \tfrac{1}{d} h(x(t))$ are equicontinuous on $(-\infty, \infty)$. Combining these two facts with the fact that (\ref{eq-lc4-2b}) holds for all $T = 1 , 2, \ldots$, it follows that $\{\bar x(t + \cdot)\}_{t \in \R}$ is equicontinuous. Therefore, given its uniform boundedness, it is relatively compact in $\C \big((-\infty, \infty); \R^{d} \big)$.
 
Now let $x^*(\cdot) \in  \C \big((-\infty, \infty); \R^{d} \big)$ be the limit of any convergent sequence $\{\bar x(t_k+ \cdot)\}_{k \geq 1}$ with $t_k \to \infty$. 
Then $\bar x(t_k) \to x^*(0)$ and $\bar x(t_k + \cdot) \to x^*(\cdot)$ uniformly on each interval $[-T, T]$, $T = 1, 2, \ldots$, as $k \to \infty$. 
With (\ref{eq-lc4-2b}) holding for all these $T$, this implies $x^k(\cdot) \to x^*(\cdot)$ in $\C \big((-\infty, \infty); \R^{d} \big)$, where $x^k(\cdot)$ is the solution of the ODE $\dot{x}(t) = \tfrac{1}{d} h(x(t))$ on $(-\infty, \infty)$ with $x^k(0) = \bar x(t_k)$. On the other hand, since $x^k(0) \to x^*(0)$, by the Lipschitz continuity of $h$ and Gronwall's inequality, $x^k(\cdot)$ also converges, uniformly on each compact interval, to the solution of the ODE $\dot{x}(t) = \tfrac{1}{d} h(x(t))$ with condition $x(0) = x^*(0)$. Therefore, $x^*(\cdot)$ must coincide with this ODE solution. Additionally, by Thm.~\ref{thm-2}(i) proved above, $\{x_n\}$ converges to some compact invariant set $D$ of this ODE. Given this and the equicontinuity of $\{ \bar x(t_k + \cdot)\}_{k \geq 1}$ proved earlier,  it follows that $\bar x(t_k)$ converges to $D$,  and hence $x^*(0) \in D$. Since $D$ is invariant, this implies $x^*(t) \in D$ for all $t \in \R$.
\end{proof}

\begin{proof}[Proof of Cor.~\ref{cor-ql}]
Under our assumptions, part (i) is implied by Thm.~\ref{thm-2}(i), since $\{x_n\}$ converges to some compact invariant set of the ODE $\dot{x}(t) = h(x(t))$ by Thm.~\ref{thm-2}(i), and $E_h$, being globally asymptotically stable, contains all compact invariant sets of this ODE (see, e.g., \cite[proof of Lem.~6.5]{WYS24}).

For part (ii), by the definition of $E_h$, if $x(\cdot)$ is a solution of the ODE $\dot{x}(t) = \tfrac{1}{d} h(x(t))$ that lies entirely in $E_h$, then $x(\cdot) \equiv x^*$ for some $x^* \in E_h$.
Therefore, by Thm.~\ref{thm-2}(ii), if $x_{n_k} \to x^* \in E_h$ as $k \to \infty$, then $\bar x(t_{n_k} + \cdot) \to x(\cdot) \equiv x^*$ in $\C \big((-\infty, \infty); \R^{d} \big)$. This means that $\bar x(t_{n_k} + s)$ converges to $x^*$ uniformly in $s$ on compact intervals. Consequently, we must have $\tau_{\delta,k} \to \infty$ as $k \to \infty$.  
\end{proof}

\section{Further Convergence Analysis via Shadowing Properties}\label{sec-shad} 
In this section, we prove Thm.~\ref{thm-3}, which establishes the convergence of the asynchronous SA algorithm~\eqref{eq-alg0} to a unique equilibrium point in $E_h$. The proof is based on an analysis of the shadowing properties of the trajectories of the algorithm~\eqref{eq-alg0}, building on prior work by Hirsch and Bena\"{i}m.

To keep the notation concise, we will not repeat `a.s.' in every instance below. When working with sample path properties, it should be understood that \emph{we are considering a sample path on which all previously established a.s.\ properties (e.g., stability and convergence) hold, along with the a.s.\ conditions required by Thm.~\ref{thm-3}. Any properties proven to hold `a.s.'\ in the analysis below will apply to this path from the point they are established.}

\subsection{Main Argument and Proof Structure for Theorem~\ref{thm-3}} \label{sec-shad1}
 
For algorithm~\eqref{eq-alg0}, we consider the interpolated trajectory $\bar x(t)$ defined by \eqref{eq-cont-traj2} and analyze its behavior at discrete moments $t = j \in\{ 0, 1, \ldots\}$. For each $j \geq 0$, let ${\tl x}^j(\cdot)$ and $x^j(\cdot)$ denote the unique solutions on the interval $[j, j+1]$ of the ODEs (\ref{eq-ode1}) and (\ref{eq-ode2}), respectively, with initial conditions ${\tl x}^j(j) = x^j(j) = \bar x(j)$.
Recall that these ODEs are given by $\dot{x}(t) = \tl \lambda(t) h(x(t))$ and $\dot{x}(t)  = \tfrac{1}{d} h(x(t))$.

Part (i) of the next proposition follows from Hirsch's proof of his shadowing theorem \cite{Hir94} (see also Bena\"{i}m \cite[Sec.\ 5]{Ben96} and \cite[Sec.\ 8.2]{Ben99}) and is the key argument in  the proof of Thm.~\ref{thm-3}. We include the proof of part (i) for completeness and clarity, as we use the Lipschitz continuity of the mappings involved here in place of the `expansion rates' of smooth ($C^1$) flows/semiflows considered in \cite{Hir94,Ben96,Ben99}. Recall that $L_h$ is the Lipschitz constant of $h$ w.r.t.\ $\| \cdot\|_\infty$. (Below, we take $\| \cdot \| = \|\cdot\|_\infty$.)

\begin{prop} \label{prp-shad1}
If $\limsup_{j \to \infty} \tfrac{1}{j} \ln \left(\| \bar x(j+1) - x^j(j+1)\| \right) < - L_h/d$, then:
\begin{itemize}[leftmargin=0.7cm,labelwidth=!]
\item[\rm (i)] There exists $z_* \in \R^d$ such that, for some $\delta \in (0, e^{- L_h/d})$ and a path-dependent constant $c$, the inequality $\| \bar x(j) - z(j) \| \leq c \delta^j$ holds for all $j \geq 0$, where $z(\cdot)$ is the solution of the limiting ODE~\eqref{eq-ode2}, $\dot{x}(t) = \tfrac{1}{d} h(x(t))$, with initial condition $z(0) = z_*$.
\item[\rm (ii)] The sequence $\{x_n\}$ from algorithm~\eqref{eq-alg0} converges to some point in $E_h$.
\end{itemize}
\end{prop}
\begin{proof}
For the limiting ODE~\eqref{eq-ode2}, let $F : \R^d \to \R^d$ be the map that takes an initial condition $y(0) = y$ to the ODE solution at time $1$. Then $F^{-1}$ maps an initial condition $y$ to the ODE solution at time $-1$ and by Gronwall's inequality, 
\begin{equation} \label{prf-shad1-1}
\|F^{-1}(y) - F^{-1}(y') \| \leq e^{L_h/d} \| y - y' \|, \quad \forall \, y, y' \in \R^d.
\end{equation}

For $j \geq 0$, denote $y_j = \bar x(j)$. The condition of the proposition implies that for some $\beta > L_h/d$ and sufficiently large $\hat j$, we have
\begin{equation} \label{prf-shad1-2}
     \| y_{j+1} - F(y_j) \| < e^{-\beta j} < e^{ - L_h j /d}, \quad \forall \, j \geq \hat j.
\end{equation}
Choose $\delta \in (0,1)$ with $e^{-\beta} < \delta < e^{ - L_h/d}$. Let $B_j\= B(y_j, \delta^j)$ denote the closed ball centered at $y_j$ with radius $\delta^j$. For $j \geq \hat j$ and $z \in B(y_{j+1}, \delta^{j+1})$, by \eqref{prf-shad1-1} and \eqref{prf-shad1-2}, 
\begin{align*}
   \| F^{-1}(z) - y_j \| & = \| F^{-1}(z) - F^{-1}(F(y_j)) \| \leq e^{L_h/d} \| z - F(y_j) \| \\
   & \leq e^{L_h/d} \left( \| z - y_{j+1} \|+ \| y_{j+1} - F(y_j) \| \right) \leq e^{L_h/d} ( \delta^{j+1} + e^{ - \beta j }),
\end{align*} 
which, in view of the choice of $\delta$, implies that $\| F^{-1}(z) - y_j \| \leq \delta^j$ for all $j$ sufficiently large. Thus, there exists $\bar j \geq \hat j$ such that $F^{-1}(B_{j+1}) \subset B_j$ for all $j \geq \bar j$. Consequently,
\begin{equation} \label{prf-shad1-3}
   F^{-i}(B_{\bar j+ i}) \subset F^{-i+1}(B_{\bar j + i -1}) \subset \cdots \subset F^{-1}(B_{\bar j+ 1}) \subset B_{\bar j}, \qquad \forall \, i \geq 0.
\end{equation}   
The intersection of these sets, $\cap_{i = 0}^\infty F^{-i}(B_{\bar j+ i})$, contains a single point $\bar z$ (since these sets form a nested sequence of nonempty compact sets, with shrinking diameters as implied by \eqref{prf-shad1-1} and the choice of $\delta$). 

From $\bar z \in \cap_{i = 0}^\infty F^{-i}(B_{\bar j+ i})$ we obtain a desired shadowing property:
\begin{equation} \label{prf-shad1-4}
  \| y_{\bar j + i} - F^i(\bar z)\| \leq \delta^{\bar j + i} \to 0, \ \ \ \text{as} \ i \to \infty. 
\end{equation}
Expressing \eqref{prf-shad1-4} in terms of $\bar x(\cdot)$ and the solution $z(t)$ to the ODE~\eqref{eq-ode2} with initial condition $z(0) = F^{- \bar j}(\bar z) =: z_*$, we have that, for some (path-dependent) constant $c$, $\| \bar x(j) - z(j) \| \leq c \, \delta^{j} \to 0$ as $j \to \infty$, which proves part (i). 

By the condition of Thm.~\ref{thm-3} on the ODE~\eqref{eq-ode2}, $z(j) \to z^*$ for some $z^* \in E_h$. Thus, the preceding inequality gives $\bar x(j) \to z^*$. Then, by Cor.~\ref{cor-ql}(ii), the sequence $\{x_n\}$ must also converge to $z^*$, proving part (ii).
\end{proof}

Clearly, Thm.~\ref{thm-3} follows immediately if we show that the condition in Prop.~\ref{prp-shad1} holds. To this end, decompose the `tracking error' of the trajectory $\bar x(\cdot)$ as
\begin{equation} \label{ineq-prf-shad}
  \| \bar x(j+ 1) - x^j(j+1) \| \leq  \| \bar x(j+ 1) - \tilde{x}^j(j+1) \| +  \| \tilde{x}^j(j+1) - x^j(j+1) \|,
\end{equation}  
separating the first component, attributable to stochastic noise and representing the error in tracking the solutions of the non-autonomous ODE~\eqref{eq-ode1}, from the second component, which represents the error due to asynchrony. In the next two subsections, we analyze each component separately and prove that, under the conditions of Thm.~\ref{thm-3}, 
\begin{align}
 \limsup_{j \to \infty} \tfrac{1}{j} \ln \left(\| \bar x(j+1) - \tilde{x}^j(j+1)\| \right) & < - L_h/d,  \label{eq-shad-cond1} \\
 \limsup_{j \to \infty} \tfrac{1}{j} \ln \left(\| \tilde{x}^j(j+1) - x^j(j+1)\| \right) & < - L_h/d. \label{eq-shad-cond2}
\end{align}
This will establish that the condition in Prop.~\ref{prp-shad1} is met and hence Thm.~\ref{thm-3} holds.

Define $\K(t) \= \max \{ n \geq 0 \mid \tilde t(n) \leq t\}$. Note that $\K(t)$ is a stopping time, since $\tilde t(n) = \sum_{k=0}^{n -1} \tl \alpha_k$ (as defined before \eqref{eq-cont-traj2}) is $\F_{n-1}$-measurable. Let 
$$ \K_1[j] \= \{ k \mid \K(j) \leq k \leq  \K(j+1) \}.$$

Recall from Rem.~\ref{rmk-shad-prf} the quantity $\ell(\{\alpha_n\}) = \limsup_{n \to \infty} \tfrac{\ln(\alpha_n)}{\sum_{k=0}^{n} \alpha_k}$ associated with the stepsize sequence $\{\alpha_n\}$. 
For the class-1 and class-2 stepsizes defined in Sec.~\ref{sec-2.2}, we have $\ell(\{\alpha_n\}) = - A$ and $- \infty$, respectively, where $A$ is the stepsize scaling parameter. Most proofs below rely only on the quantity $\ell(\{\alpha_n\})$ and the nonincreasing property of the stepsizes, without assuming  specific stepsize forms---except near the end of the proof of \eqref{eq-shad-cond2} (in Lem.~\ref{lem-shad3-2}, Sec.~\ref{sec-A.1.2}). We first prove \eqref{eq-shad-cond1}.

\subsection{Proof for \eqref{eq-shad-cond1}} \label{sec-A.1.1}
Recall the equivalent form \eqref{eq-alg1} of algorithm \eqref{eq-alg0} in vector notation:
\begin{equation} \label{eq-alg1a}
    x_{n+1}  = x_n + \tl \alpha_n  \tl \Lambda_n \left( h (x_n) + M_{n+1} + \epsilon_{n+1} \right),
\end{equation}
where $\tl \Lambda_n$ is a $d \times d$-diagonal matrix given by $\tl \Lambda_n \=  \text{diag} \big( \tl \q(n, 1), \tl \q(n, 2), \ldots, \tl \q(n, d) \big)$ with diagonal entries $\tl \q(n, i) = \frac{\alpha_{\nu(n, i)}}{\tl \alpha_n} \ind\{i \in Y_n\} \in [0,1]$. (Note that $\sum_n \tl \alpha_n = \infty$ and $\sum_n {\tl \alpha}^2_n < \infty$ a.s.\ by Assums.~\ref{cond-ss}(i) and~\ref{cond-us}(i).) 

By \eqref{eq-alg1a} and the definitions of $\bar x(\cdot)$ and $\tilde{x}^j(\cdot)$, we can bound $\| \bar x(j+1) - \tilde{x}^j(j+1)\|$ using similar derivations as in the proof of \cite[Lem.\ 2.1]{Bor23}, yielding that for all $j \geq 0$, 
\begin{equation}  \label{eq-shad2-0a}
   \| \bar x(j+1) - \tilde{x}^j(j+1)\| \leq c_0 \sup_{k \in \K_1[j]} \tilde \alpha_k + c_1 D_1(j) + c_2 D_2(j).
\end{equation}
Here, the $c_i$'s are constants that depend on the sample path but not on $j$, and 
\begin{equation}    \label{eq-shad2-0b}
D_1(j) \= \sup_{k \in \K_1[j]} \Big\| \sum_{i=\K(j)}^k \tilde \alpha_i \tilde \Lambda_i \epsilon_{i+1} \Big\|, \qquad D_2(j) \= \sup_{k \in \K_1[j]} \Big\| \sum_{i=\K(j)}^k \tilde \alpha_i \tilde \Lambda_i M_{i+1} \Big\|.
\end{equation}
Before proceeding to bound the three terms on the right-hand side (r.h.s.)\ of \eqref{eq-shad2-0a}, we derive some limiting properties of the stepsizes that will be needed in our analysis.

\begin{lemma} \label{lem-shad2} The following hold a.s.:\\
\rm{(i)} For all $i' \in \I$, $\lim_{j \to \infty} \tfrac{1}{j} \sum_{k=0}^{\K(j)} \alpha_{\nu(k, i')} \ind \{ i' \in Y_k\} = \lim_{j \to \infty} \tfrac{1}{j} \sum_{k=0}^{\K(j)} \alpha_k = \tfrac{1}{d}$.\\
\rm{(ii)} $\limsup_{j \to \infty} \tfrac{1}{j} \ln (\alpha_{\K(j)}) \leq \tfrac{\ell(\{\alpha_n\})}{d}$.\\
\rm{(iii)} $\limsup_{j \to \infty} \tfrac{1}{j} \ln \big(\sum_{i' \in \I} \alpha_{\nu(\K(j), i')} \big) \leq \tfrac{\ell(\{\alpha_n\})}{d}$.
\end{lemma}

\begin{proof}
Part (i) follows from Assums.~\ref{cond-ss}(iii),~\ref{cond-us}(i) and the definition of the ODE-time $j$ via the stepsizes $\{\tilde \alpha_n\}$. Specifically, recall the constant $\Delta \in (0,1]$ from Assum.~\ref{cond-us}(i), and let $\Delta' \in (0, \Delta)$. By Assum.~\ref{cond-ss}(iii), as $j \to \infty$, $\tfrac{\sum_{k=0}^{[y\K(j)]} \alpha_{k}}{\sum_{k=0}^{\K(j)} \alpha_k} \to 1$ uniformly in $y \in [\Delta', 1]$. Combined with Assum.~\ref{cond-us}(i), this yields $\tfrac{\sum_{k=0}^{\K(j)} \alpha_{\nu(k, i')}\ind \{ i' \in Y_k\}}{\sum_{k=0}^{\K(j)} \alpha_k} \to 1$ for all $i' \in \I$. Since 
$$\lim_{j \to \infty} \textstyle{\big|j - \sum_{i' \in \I} \sum_{k=0}^{\K(j)} \alpha_{\nu(k, i')} \ind \{ i' \in Y_k\} \big|} = 0,$$ 
part (i) follows.
By part (i), 
$$\limsup_{j \to \infty} \tfrac{1}{j} \ln (\alpha_{\K(j)}) = \limsup_{j \to \infty} \tfrac{\ln (\alpha_{\K(j)})}{ d \, \sum_{k=0}^{\K(j)} \alpha_k} \leq \tfrac{\ell(\{\alpha_n\})}{d},$$ 
proving part (ii).

By Assum.~\ref{cond-ss}(ii), we have that for some constant $c$, $\tfrac{\alpha_{[\Delta'\K(j)]}}{\alpha_{\K(j)}} \leq c$ for all $j \geq 0$. Since $\{\alpha_n\}$ is nonincreasing, we get from this inequality and Assum.~\ref{cond-us}(i) that for all sufficiently large $j$,
$\sum_{i' \in \I} \alpha_{\nu(\K(j), i')} \leq \sum_{i' \in \I}  \alpha_{[\Delta'\K(j)]} \leq d c \alpha_{\K(j)}$. Part (iii) now follows from part (ii).
\end{proof}

Bounding the first two terms on the r.h.s.\ of \eqref{eq-shad2-0a} is straightforward:
\begin{lemma} \label{lem-shad3-0}
Almost surely, with $\mu_\delta$ as given in Assum.~\ref{cond-mns},
\begin{equation} \label{eq-shad3-0}
    \limsup_{j \to \infty}  \tfrac{1}{j} \ln \Big(\sup_{k \in \K_1[j]} \tilde \alpha_k \Big) \leq \tfrac{\ell(\{\alpha_n\})}{d}, \qquad \limsup_{j \to \infty}  \tfrac{1}{j} \ln \big(D_1(j)\big) \leq  \tfrac{\mu_\delta}{d}.
\end{equation}
\end{lemma}

\begin{proof}
Since $\sup_{k \in \K_1[j]} \tilde \alpha_k \leq \sum_{i' \in \I} \alpha_{\nu(\K(j), i')}$ (due to $\{\alpha_k\}$ being nonincreasing), Lem.~\ref{lem-shad2}(iii) gives the first bound in \eqref{eq-shad3-0}. From Assum.~\ref{cond-ns}(ii) on $\{ \epsilon_{k}\}$, the a.s.\ boundedness of $\{x_n\}$ (Thm.~\ref{thm-1}), and the fact $\sum_{i = \K(j)}^{\K(j+1)} \tilde \alpha_i \leq 1 + 2d\sup_n \alpha_n$, it follows that $D_1(j) \leq c \sup_{k \in \K_1[j]} \delta_{k+1}$ for some (sample path-dependent) constant $c$. Then, by Assum.~\ref{cond-mns} on $\{\delta_k\}$ and Lem.~\ref{lem-shad2}(i), we obtain
$\limsup_{j \to \infty}  \tfrac{1}{j} \ln \big(D_1(j)\big) =  \limsup_{j \to \infty}  \tfrac{\ln \big(D_1(j)\big)}{d \sum_{k=0}^{\K(j)} \alpha_k} \leq  \tfrac{\mu_\delta}{d}$, proving the second bound in \eqref{eq-shad3-0}.
\end{proof}

We now bound the third term $D_2(j)$ in \eqref{eq-shad2-0a} similarly to Bena\"{i}m \cite[proofs of Props.\ 4.2, 8.3]{Ben99}, with additional effort to account for the random stepsizes $\{ \tl \alpha_i\}$ and more general noise terms $\{M_{i} \}$ in our setting.

\begin{lemma} \label{lem-shad3}
Almost surely, $\limsup_{j \to \infty} \tfrac{1}{j} \ln \big( D_2(j) \big) \leq \frac{\ell(\{\alpha_n\})}{2d}$.
\end{lemma} 
\begin{proof}
For integers $N \geq 1$, define stopping times $\tau_N$ and auxiliary variables $M^{(N)}_{i}$:
$$\tau_N \= \min \{ i \geq 0: \| x_i \| > N \ \text{or} \ K_i > N \}, \qquad M^{(N)}_{i+1} \= \ind \{ i < \tau_N \} M_{i+1},  \ \ \ i \geq 0,$$
where $\{K_i\}$ is the a.s.\ bounded random sequence in Assum.~\ref{cond-ns}(i).
By Assum.~\ref{cond-ns}(i), for each $N$, $\{M^{(N)}_i \}_{i \geq 1}$ is a martingale-difference sequence satisfying $\sup_{i} \E[ \| M^{(N)}_{i+1} \|^2_2 \mid \F_i ] \leq c_N$ for some constant $c_N$ depending on $N$. Since $\{x_i\}$ is bounded a.s.\ by Thm.~\ref{thm-1}, almost surely, $\{M_i\}$ coincides with $\{M^{(N)}_{i}\}$ for some sample path-dependent value of $N$. Thus, to prove the lemma, it suffices to prove that a.s.\ for \emph{each $N$}, 
\begin{equation} \label{eq-prf-shad2-0}
  \limsup_{j \to \infty} \tfrac{1}{j} \ln \big( \tilde D_2(j) \big) \leq \tfrac{\ell(\{\alpha_n\})}{2d}, \quad \text{where} \  \tilde D_2(j) \= \sup_{k \in \K_1[j]} \Big\| \sum_{i=\K(j)}^k \tilde \alpha_i \tilde \Lambda_i M^{(N)}_{i+1} \Big\|.
\end{equation}  

To this end, fix $N$ and let $T = 1 + d \sup_n a_n$. For each $n$, define
$$\zeta_{n,k} \= \sum_{i = n}^{k-1} \tilde \alpha_i \cdot \ind\{\tilde t(i) \leq \tilde t(n) + T\} \cdot \tilde \Lambda_i M ^{(N)}_{i+1} \ \ \text{for} \  k \geq n, \ \ \text{with} \ \zeta_{n,n} \= 0.$$
Then $\{\zeta_{n,k}\}_{k \geq n}$ is a square-integrable martingale w.r.t.\ $\{\F_k\}_{k \geq n}$, and
\begin{equation} \label{eq-prf-shad2-1}
\lim_{\bar k \to \infty} \E \Big[  \sup_{k \leq \bar k} \| \zeta_{n,k}\|^2_2 \, \Big| \, \F_n \Big] =  \E \left[ \sup_{k: \, \tilde t(n) \leq \tilde t(k) \leq \tilde t(n) + T} \Big\| \sum_{i =n}^k \tilde \alpha_i \tilde \Lambda_i M^{(N)}_{i+1} \Big\|^2_2 \,\, \Bigg| \, \F_n \right]  
\end{equation}
by the monotone convergence theorem. For $\bar k > n$, applying Doob's and Burkholder's inequalities \cite[Thms.\ 2.2, 2.10]{HaH80}, we obtain that for some constant $c_0$ (independent of the specific problem data),
\begin{align}
  \E \Big[  \sup_{k \leq \bar k} \| \zeta_{n,k}\|^2_2 \mid \F_n \Big] & \leq c_0 \, \E \Big[ \sum_{i = n}^{\bar k - 1}  \left\| \tilde \alpha_i \, \ind\{\tilde t(i) \leq \tilde t(n) + T\} \, \tilde \Lambda_i M ^{(N)}_{i+1} \right\|^2_2 \, \Big| \, \F_n \Big] \notag \\
  & \leq c_0 \, \E \Big[ \sum_{i = n}^{\bar k - 1} {\tilde \alpha}^2_i \, \ind\{\tilde t(i) \leq \tilde t(n) + T\} \, \| \tilde \Lambda_i\|^2_2 \cdot \E \big[ \| M ^{(N)}_{i+1} \|^2_2 \, \big| \, \F_i \big] \, \Big| \, \F_n \Big] \notag \\
  & \leq c_0 c_1 c_N \, \E \Big[ \sum_{i = n}^{\bar k - 1} {\tilde \alpha}^2_i \, \ind\{\tilde t(i) \leq \tilde t(n) + T\} \Big| \, \F_n \Big]  \label{eq-prf-shad2-2}\\
  & \leq c_2 \, \E \Big[  \Big(\sum_{i=n}^{\K(\tilde t(n) + T)} \tilde \alpha_i \Big)  \cdot \sup_{n \leq i \leq \K(\tilde t(n) + T)} \tilde \alpha_i \,\Big| \, \F_n \Big] \leq c_3 \, \sum_{i' \in \I} \alpha_{\nu(n, i')}. \label{eq-prf-shad2-3}
\end{align} 
Here, $c_1$ is a deterministic bound on $\sup_{i \geq 0} \| \tilde \Lambda_i\|^2_2$. To derive \eqref{eq-prf-shad2-2}, we also used the bound $\sup_{i \geq 0} \E \big[ \| M ^{(N)}_{i+1} \|^2_2 \mid \F_i \big]  \leq c_N$. In \eqref{eq-prf-shad2-3}, $c_2$ and $c_3$ are deterministic constants independent of $n$ and $\bar k$, specifically $c_2 = c_0 c_1 c_N$ and $c_3 = c_2 (T + d \sup_{i} \alpha_i)$. In the last inequality in \eqref{eq-prf-shad2-3}, we used the fact that $\sum_{i=n}^{\K(\tilde t(n) + T)} \tilde \alpha_i \leq T + d \sup_{i} \alpha_i$ and the stepsizes are nonincreasing. 

Combining \eqref{eq-prf-shad2-1} and \eqref{eq-prf-shad2-3}, we obtain
$$ \E \left[ \sup_{k: \, \tilde t(n) \leq \tilde t(k) \leq \tilde t(n) + T} \Big\| \sum_{i =n}^k \tilde \alpha_i \tilde \Lambda_i M^{(N)}_{i+1} \Big\|^2_2 \,\, \Bigg| \, \F_n \right] \leq c_3 \, \sum_{i' \in \I} \alpha_{\nu(n, i')}, \qquad \forall \, n \geq 0.$$
Then by \cite[Prop.\ II-1-3]{Nev75}, the above inequality holds for any stopping time $\tau$ (w.r.t.\ $\{\F_n\}$), with $\tau$ in place of $n$ and with $\F_\tau$ in place of $\F_n$. Consequently, for the stopping times $\K(j), j \geq 1$, we obtain the bounds
\begin{equation} \label{eq-prf-shad2-4}
\E \left[ \sup_{k \in \K_1[j]} \Big\| \sum_{i =\K(j)}^k \tilde \alpha_i \tilde \Lambda_i M^{(N)}_{i+1} \Big\|^2_2 \, \Big| \, \F_{\K(j)} \right] \leq c_3 \, \sum_{i' \in \I} \alpha_{\nu(\K(j), i')}, \quad \forall \, j \geq 1.
\end{equation} 
where we have also used the fact that $\K_1[j] \subset \{ k : \, \tilde t(\K(j)) \leq \tilde t(k) \leq \tilde t(\K(j)) + T\}$.

Finally, to prove \eqref{eq-prf-shad2-0}, consider any $\beta > 0$ with $\beta + \tfrac{\ell(\{ \alpha_n\})}{2 d} < 0$. Since $\|x\|_\infty \leq \|x\|_2$, applying \eqref{eq-prf-shad2-4}, Markov's inequality, and Lem.~\ref{lem-shad2}(iii), we obtain
\begin{equation} \label{eq-prf-shad2-5}
 \sum_{j=1}^\infty \Pr \left( \tilde D_2(j) \geq e^{- \beta j}  \,\Big|\, \F_{\K(j)} \right) \leq c_3 \, \sum_{j=1}^\infty e^{ 2 \beta j} \cdot \Big(\sum_{i' \in \I} \alpha_{\nu(\K(j), i')} \Big) < \infty.
\end{equation}
By the extended Borel--Cantelli lemma \cite[Cor.\ 2.3]{HaH80}, this implies that, almost surely, $\tilde D_2(j) < e^{- \beta j}$ for all sufficiently large $j$.%footnote starts
\footnote{Details: To apply the extended Borel--Cantelli lemma \cite[Cor.\ 2.3]{HaH80}, take an integer $m \geq 1 + d \sup_n \alpha_n$ and consider separately each subsequence $\{\tilde D_2 (i+ k m)\}_{k \geq 0}$ for $i = 1, \ldots, m$.  Note that $\tilde D_2(j)$ is $\F_{\K(j+1)+1}$-measurable by definition, while $\K(j+1)+1 \leq \K(j+m)$. Therefore, $\tilde D_2 (j)$ is $\F_{\K(j+m)}$-measurable. Consequently, for each $i \leq m$, \cite[Cor.\ 2.3]{HaH80} applies to the subsequence $\{\tilde D_2 (i+ k m)\}_{k \geq 0}$, and \eqref{eq-prf-shad2-5} implies $\tilde D_2(i + km) < e^{- \beta (i+km)}$ for all sufficiently large $k$.}
%footnote ends
Letting $\beta \uparrow - \tfrac{\ell(\{ \alpha_n\})}{2 d}$ then proves \eqref{eq-prf-shad2-0}.
\end{proof}

By \eqref{eq-shad2-0a}-\eqref{eq-shad2-0b} and Lems.\ \ref{lem-shad3-0} and \ref{lem-shad3}, inequality \eqref{eq-shad-cond1} holds a.s., namely,
$$  \limsup_{j \to \infty} \tfrac{1}{j} \ln \left(\| \bar x(j+1) - \tilde{x}^j(j+1)\| \right)  < - L_h/d \ \ a.s.,$$
if $\max\big\{ \tfrac{\ell(\{\alpha_n\})}{2}, \, \mu_\delta \big\} < - L_h$. This condition becomes $\max\big\{ \tfrac{-A}{2}, \, \mu_\delta \big\} < - L_h$ for class-1 stepsizes and $\mu_\delta < - L_h$ for class-2 stepsizes. Since these are the conditions required in Thm.~\ref{thm-3}, it follows that \eqref{eq-shad-cond1} holds.

\subsection{Proof for \eqref{eq-shad-cond2}} \label{sec-A.1.2}
We now proceed to establish \eqref{eq-shad-cond2}. Since $\tilde{x}^j(\cdot)$ and $x^j(\cdot)$ are solutions of the ODEs \eqref{eq-ode1} and \eqref{eq-ode2}, respectively, with $\tilde{x}^j(j)=x^j(j) = \bar x(j)$, we have that for $t \in [j, j+1]$, 
\begin{align*}
  \tilde{x}^j(t) - x^j(t)  & = \int_j^t \tilde \lambda(\tau) h(\tilde{x}^j(\tau)) d \tau - \int_j^t \tfrac{1}{d} h(x^j(\tau)) d \tau \\
   & = \int_j^t \big(\tilde \lambda(\tau) - \tfrac{1}{d} I\big) h(x^j(\tau)) d \tau + \int_j^t \tilde \lambda(\tau) \big(h(\tilde{x}^j(\tau)) - h(x^j(\tau)) \big) d\tau.
\end{align*}     
As the trajectory $\tilde \lambda(\cdot)$ is bounded by definition \eqref{eq-tlambda} and $h$ is Lipschitz continuous, applying Gronwall's inequality yields that for some constant $c$ independent of $j$, 
\begin{equation} \label{eq-shad3-d}
 \| \tilde{x}^j(j+1) - x^j(j+1) \| \leq c \, \sup_{s \in [0,1]} \left\| \int_0^s \big(\tilde \lambda(j+\tau) - \tfrac{1}{d} I\big) h(x^j(j + \tau))) d \tau \right\|.
\end{equation} 

For each $i \in \I$, let $F^j_i(\cdot)$ denote the distribution function of the measure with density $\tilde \lambda_{ii}(j + \cdot)$ on $[0,1]$; that is, $F^j_i(s) = \int_{0}^s \tilde \lambda_{ii}(j + \tau) d\tau$, $s \in [0,1]$.

\begin{lemma} \label{lem-shad3-1}
Almost surely, for some (sample path-dependent) constant $c$, 
\begin{equation} \label{eq-shad3-1}
  \| \tilde{x}^j(j+1) - x^j(j+1) \|  \leq c \sup_{i \in \I} \sup_{s \in [0,1]} \big| F^j_i(s) - \tfrac{s}{d} \big|, \quad \  \forall \, j \geq 0.
\end{equation}  
\end{lemma}
\begin{proof}
First, observe that on $[0,1]$, the family of functions $h(x^j(j + \cdot))$ for $j \geq 0$ is uniformly bounded and uniformly Lipschitz continuous with a common Lipschitz constant that depends on the sample path. (This follows from the Lipschitz continuity of $h$, the definition of the ODE solutions $x^j(\cdot)$, and the boundedness of the iterates $\{x_n\}$.) Now, for a fixed $j \geq 0$ and for each $i \in \I$, consider the $i$th component of the integral term in \eqref{eq-shad3-d}: $\psi(s) \= \int_0^s \big(\tilde \lambda_{ii}(j+\tau) - \tfrac{1}{d} \big) g_i(\tau) d \tau$, where $g_i(\tau) \= h_i(x^j(j + \tau))$.
By the integration by parts formula for the absolutely continuous functions $g_i (\tau)$ and $F^j_i(\tau) - \tfrac{\tau}{d}$, we have
\begin{equation} \label{eq-shad3-2}
  \psi(s) = g_i(s) (F^j_i(s) - \tfrac{s}{d}) - \int_{0}^s \!\big(F^j_i(\tau) - \tfrac{\tau}{d} \big) g'_i(\tau) d\tau, \qquad s \in [0,1],
\end{equation}  
where the derivative $g'_i(\tau)$ is defined a.e.\ and bounded by the Lipschitz constant of $g_i$. In view of the uniform boundedness and uniform Lipschitz continuity of the family of functions $\{h(x^j(j + \cdot))\}_{j \geq 0}$ on $[0,1]$, the expression~\eqref{eq-shad3-2} implies that for some (path-dependent) constant $c$ independent of $j$, $|\psi(s)| \leq c \sup_{\tau \in [0, s]} |F^j_i(\tau) - \tfrac{\tau}{d}|$. This, together with inequality \eqref{eq-shad3-d}, leads to the desired bound \eqref{eq-shad3-1} for all $j$.
\end{proof}

Lemma~\ref{lem-cvg-2} implies that for each $i \in \I$, as $j \to \infty$, the measure on $[0,1]$ with density $\tilde \lambda_{ii}(j + \cdot)$ converges setwise to $\tfrac{1}{d}$ times the Lebesgue measure. Therefore, the r.h.s.\ of \eqref{eq-shad3-1} converges to $0$ as $j \to \infty$. The next lemma addresses the speed of this convergence---sufficiently rapid convergence is needed for inequality \eqref{eq-shad-cond2} to hold.

\begin{lemma} \label{lem-shad3-2} Almost surely,
\begin{equation} \label{eq-prf-shad3-1}
\limsup_{j \to \infty} \tfrac{1}{j} \ln \!\Big(\sup_{s \in [0,1]} \big| F^j_i(s) - \tfrac{s}{d} \big| \Big) \leq \tfrac{\ell}{d}, \quad \ \forall \, i \in \I,
\end{equation}
where $\ell = - (\gamma\wedge 1) A$ for class-1 stepsizes (with $\gamma$ being the constant in Assum.~\ref{cond-mus}(i)), and $\ell = - A$ for class-2 stepsizes (where $A$ is the stepsize scaling factor in both cases).
\end{lemma}

\begin{proof}
By definition the distribution functions $F^j_i$ satisfy that for $s \in [0,1]$,
\begin{align}
    F^j_i(s)  = \int_0^s \tilde \lambda_{ii} (j + \tau) d\tau, \qquad \sum_{i \in \I} F^j_i(s)  = \int_0^s \sum_{i \in \I} \tilde \lambda_{ii} (j + \tau) d\tau = s \label{eq-prf-shad3-2}
\end{align}
(since $\sum_{i \in \I} \tilde \lambda_{ii}(\cdot) \equiv 1$ by the definition of $\tilde \lambda(\cdot)$; see \eqref{eq-tlambda}).
Hence, for each $\bar i \in \I$, 
\begin{equation} \label{eq-prf-shad3-2b}
   F^j_{\bar i}(s) - \tfrac{s}{d} = F^j_{\bar i}(s) - \tfrac{1}{d} \sum_{i \in \I} F^j_{i}(s)  =  F^j_{\bar i}(s) - G^j(s) - \tfrac{1}{d} \sum_{i \in \I} \big( F^j_{i}(s) - G^j(s) \big),
\end{equation}   
where $G^j(s) \= \int_{\K(j)}^{\K(j+s)} \!g_\alpha(y) dy$, with $g_\alpha(y) = \tfrac{1}{Ay}$ for class-1 $\{\alpha_n\}$ and  $g_\alpha(y) = \tfrac{1}{A y \ln y}$ for class-2 $\{\alpha_n\}$. From \eqref{eq-prf-shad3-2b} we obtain
\begin{equation} \label{eq-prf-shad3-3}
 \big| F^j_{\bar i}(s) - \tfrac{s}{d} \big| \leq 2 \sup_{i \in \I} \big| F^j_{i}(s) - G^j(s) \big|, \quad s \in [0,1].
\end{equation}

For large $j$, we have $F^j_i(s) \approx \sum_{k=\nu(\K(j), i)}^{\nu(\K(j+s), i)} \alpha_k \approx \int_{\nu(\K(j), i)}^{\nu(\K(j+s),i)} \!g_\alpha(y) dy$. Specifically, the difference in the second approximation is bounded by $\alpha_{\nu(\K(j), i)}$. The difference in the first approximation, based on the first relation in \eqref{eq-prf-shad3-2} and the definitions of the ODE-time and $\tilde \lambda(\cdot)$ (see \eqref{eq-tlambda}), can be bounded by
$$ \textstyle{ \Big| F^j_i(s) - \sum_{k=\nu(\K(j), i)}^{\nu(\K(j+s), i)} \alpha_k \Big| \leq  \alpha_{\nu(\K(j), i)} + \alpha_{\nu(\K(j+s), i)}}.$$
Since $\{\alpha_n\}$ is nonincreasing, $\alpha_{\nu(\K(j+s), i)} \leq \alpha_{\nu(\K(j), i)}$ and moreover, by Assums.~\ref{cond-ss}(ii) and~\ref{cond-us}(i), for sufficiently large $j$, $\alpha_{\nu(\K(j), i)} \leq c \,\alpha_{\K(j)}$ for some constant $c$ independent of $j$. It then follows that for sufficiently large $j$, for all $i \in \I$ and $s \in [0,1]$, 
\begin{equation} \label{eq-prf-shad3-4}
 \big| F^j_{i}(s) - G^j(s) \big| \leq c \, \alpha_{\K(j)} + \Big| \, \underset{e_i(j, s)}{\underbrace{\int_{\nu(\K(j), i)}^{\nu(\K(j+s),i)} \!g_\alpha(y) dy - \int_{\K(j)}^{\K(j+s)} g_\alpha(y) dy}} \, \Big|,
\end{equation}    
where $c$ is a constant independent of $j$. 

To establish the desired inequality \eqref{eq-prf-shad3-1}, we now bound the r.h.s.\ of \eqref{eq-prf-shad3-4} separately for each class of stepsizes considered.

\noindent {\it Case 1: $g_\alpha(y) = \tfrac{1}{A y}$.} We have $e_i(j,s) = \tfrac{1}{A} \ln \Big(\tfrac{\nu(\K(j+s),i)}{\K(j+s)} \Big) - \tfrac{1}{A} \ln \Big(\tfrac{\nu(\K(j),i)}{\K(j)} \Big)$. Under Assum.~\ref{cond-mus}, a direct calculation shows that for sufficiently large $j$, $|e_i(j,s) | \leq c\, \K(j)^{-\gamma}$ for some (path-dependent) constant $c$ independent of $j$ and $s$. We also have $\alpha_{\K(j)} = \tfrac{1}{A \K(j)}$ and $\limsup_{j \to \infty} \tfrac{1}{j} \ln(\alpha_{\K(j)}) \leq \tfrac{\ell(\{\alpha_n\})}{d} = \tfrac{-A}{d}$ by Lem.~\ref{lem-shad2}(ii). Combining these relations with \eqref{eq-prf-shad3-3} and \eqref{eq-prf-shad3-4} yields that
$$ \limsup_{j \to \infty} \tfrac{1}{j} \ln \!\Big(\sup_{s \in [0,1]} \big| F^j_i(s) - \tfrac{s}{d} \big| \Big) \leq   \tfrac{- (\gamma\wedge 1) A}{d}, \qquad \forall \, i \in \I,$$ 
proving inequality \eqref{eq-prf-shad3-1} with $\ell = - (\gamma\wedge 1) A$.

\noindent {\it Case 2: $g_\alpha(y) = \tfrac{1}{A y \ln y}$.} Then
$e_i(j,s) = \tfrac{1}{A} \ln \Big(\tfrac{\ln \big(\nu(\K(j+s),i)\big)}{\ln(\K(j+s))} \Big) - \tfrac{1}{A} \ln \Big(\tfrac{ \ln \big( \nu(\K(j),i)\big)}{\ln (\K(j))} \Big)$. Expressing the two numerators as $\ln \big(\tfrac{\nu(\K(j+s),i)}{\K(j+s)}\big) +\ln(\K(j+s))$ and $\ln \big(\tfrac{\nu(\K(j),i)}{\K(j)}\big) +\ln(\K(j))$, respectively, and using Assum.~\ref{cond-us}(i), a direct calculation shows that for sufficiently large $j$, we have $|e_i(j,s) | \leq \tfrac{c}{\ln(\K(j))}$ for some (path-dependent) constant $c$ independent of $j$ and $s$. We also have $\lim_{j \to \infty} \tfrac{d \ln \ln (\K(j))}{A \, j} = 1$ by Lem.~\ref{lem-shad2}(i) and the definition of stepsizes in this case, while $\limsup_{j \to \infty} \tfrac{1}{j} \ln(\alpha_{\K(j)}) \leq \tfrac{\ell(\{\alpha_n\})}{d} = -\infty$ by Lem.~\ref{lem-shad2}(ii).
Combining these relations with \eqref{eq-prf-shad3-3} and \eqref{eq-prf-shad3-4} yields that
$$ \limsup_{j \to \infty} \tfrac{1}{j} \ln \!\Big(\sup_{s \in [0,1]} \big| F^j_i(s) - \tfrac{s}{d} \big| \Big) \leq \lim_{j \to \infty} \tfrac{- \ln \ln \big( \K(j)\big)}{\tfrac{d}{A} \ln \ln \big( \K(j)\big)} =  \tfrac{- A}{d}, \qquad \forall \, i \in \I,$$
establishing the desired inequality \eqref{eq-prf-shad3-1} with $\ell = - A$.
\end{proof}

By Lems.~\ref{lem-shad3-1} and~\ref{lem-shad3-2}, we conclude that inequality \eqref{eq-shad-cond2} holds a.s., namely, 
$$ \limsup_{j \to \infty} \tfrac{1}{j} \ln \left(\| \tilde{x}^j(j+1) - x^j(j+1)\| \right) < - L_h/d \ \ a.s.,$$
if $(\gamma\wedge 1) A > L_h$ when class-1 stepsizes are used and $A > L_h$ when class-2 stepsizes are used. As these are the conditions required in Thm.~\ref{thm-3}, we obtain \eqref{eq-shad-cond2}, thereby completing the proof of Thm.~\ref{thm-3}.

\section{Discussion}  \label{sec-conc-rmks}

In this paper, we established the stability and convergence of a class of asynchronous SA algorithms under more general noise conditions than previously considered. Our stability analysis extends a method of Borkar and Meyn, resolving open questions about the stability of such algorithms in average-reward RL. To sharpen the convergence analysis, we developed new characterizations of the shadowing properties of asynchronous SA by building on a dynamical systems approach of Hirsch and Bena\"{i}m. These results provide a theoretical foundation for the analysis of an important class of RL algorithms based on relative value iteration (see \cite{WYS24,YWS25} for related developments).

Although our focus has been on partially asynchronous update schemes relevant to average-reward RL, certain techniques—particularly the construction of auxiliary scaled processes using stopping arguments—may prove useful for broader classes of asynchronous schemes, including those discussed in \cite{Bor98}, given suitable choices of the function $h$. Future work may explore these possibilities, as well as potential extensions to distributed computation frameworks with communication delays.

\appendix
\counterwithin{equation}{section}
\renewcommand{\theequation}{A.\arabic{equation}}
\counterwithin{mylemma}{section}
\renewcommand{\themylemma}{A.\arabic{mylemma}}
\counterwithin{myassumption}{section}
\renewcommand{\themyassumption}{A.\arabic{myassumption}}

\titleformat{\section}{\normalfont\Large\bfseries}{\appendixname:}{1em}{}
\renewcommand{\thesection}{}

\section{Alternative Stability Proof under a Stronger\\ Noise Condition} \label{app-alt-stab}

In this appendix, we consider a stronger condition from Borkar \cite{Bor98} on the martingale-difference noise sequence $\{M_{n}\}$, and give an alternative, simpler proof of the stability theorem for this case. Below, the norm $\| \cdot\|$ denotes $\| \cdot\|_2$, in line with \cite{Bor23}, from which we draw several results.

\begin{myassumption}[Alternative condition on $\{M_n\}$] \label{cond-alt-ns} \hfill \\
For all $n \geq 0$, $M_{n+1}$ is given by $M_{n+1} = F (x_{n}, \zeta_{n+1})$, where:
\begin{enumerate}
\item[\rm (i)] $\zeta_1, \zeta_2, \ldots$ are exogenous, i.i.d.\ random variables taking values in a measurable space $\Z$, with a common distribution $p$. 
\item[\rm (ii)] The function $F : \R^d \times \Z \to \R^d$ has these properties: It is uniformly Lipschitz continuous in its first argument; i.e., for some constant $L_F > 0$,
\begin{equation}   \notag 
\| F( x, z) - F(y, z) \| \leq L_F \| x - y \|, \quad  \forall \, x, y \in \R^d, \   z \in \Z.
\end{equation} 
It is measurable in its second argument and moreover,
$$  \int_\Z \| F(0, z) \|^2 \, p(dz) < \infty, \qquad \int_\Z  F(x, z) \, p(dz) = 0,  \ \ \ \forall \, x \in \R^d.$$ 
\end{enumerate}
\end{myassumption}

Assumption~\ref{cond-alt-ns} implies Assum.~\ref{cond-ns}(i). Indeed, using the properties of the function $F$, a direct calculation shows that for some constant $K_F > 0$,
\begin{equation} \label{eq-alt-prf0}
       \int _\Z \| F(x, z) \|^2 \, p(dz) \leq K_F \!\left( 1 + \| x\|^2 \right), \quad \forall \, x \in  \R^d.
\end{equation} 
Thus, with $\F_n \= \sigma(x_m, Y_m, \zeta_m, \epsilon_m; m \leq n)$, $\{M_{n+1}\}$ satisfies Assum.~\ref{cond-ns}(i) with
\begin{equation} \label{eq-alt-prf0b}
       \E [ \| M_{n+1} \|^2 \mid \F_n ] \leq K_F \!\left( 1 + \| x_n\|^2 \right), \quad n \geq 0.
\end{equation}

By leveraging the specific form of $\{M_{n+1}\}$, we simplify the proof of the stability theorem. In this case, unlike the previous analysis in Sec.~\ref{sec-stab}, we work with the linearly interpolated trajectory $\bar x(t)$ defined in \eqref{eq-cont-traj2}, where the iterate $x_n$ is placed at the `ODE-time' $\tl t(n) = \sum_{k=0}^{n-1} \tl \alpha_k$, with the random stepsizes $\tl \alpha_k = \sum_{i \in Y_k}  \alpha_{\nu(k, i)}$ representing the elapsed times between consecutive iterates. In the same manner as before, we divide the time axis into intervals of approximately length $T$ for a given $T > 0$, and we define the scaled trajectory $\hat x(t)$ accordingly. In particular, $T_n$ and $m(n)$ are recursively defined by (\ref{eq-def-tm}), but with $\tl t(m)$ replacing $t(m)$: 
$$ m(0)= T_0 = 0 \ \ \ \text{and} \ \ \   m(n+1) \= \min \{ m : \tl t(m) \geq T_n + T \}, \ \  T_{n+1} \= \tl t\big(m(n+1)\big), \ \ n \geq 0.$$
Observe that \emph{$T_n$ and $m(n)$ are now random variables.} With $r(n) \= \| x_{m(n)} \| \vee 1$, we then define the scaled trajectory $\hat x(t)$ and a `copy' of it, $\hat x^n(t)$, on each closed internal $[T_n, T_{n+1}]$ by (\ref{eq-hx0}) and (\ref{eq-hx0-cpy}). 

As discussed in Sec.~\ref{sec-stab-prf1}, a key step in the stability analysis is to establish $\sup_t \| \hat x(t)\| < \infty$ a.s. We will now proceed to prove this.

For $m(n) \leq k < m(n+1)$, we can express $\hat x^n(\tl t(k+1))$ as
\begin{equation}  \label{eq-alt-hx}
  \hat x^n( \tl t(k+1)) = \hat x^n(\tl t(k)) + \tl \alpha_k \tl \Lambda_k h_{r(n)}(\hat x^n( \tl t(k))) +  \tl \alpha_k \tl \Lambda_k \hat M_{k+1} + \tl \alpha_k \tl \Lambda_k \hat \epsilon_{k+1},
\end{equation}  
where $\tl \Lambda_k$ is the diagonal matrix defined below (\ref{eq-alg1a0}):
$$\tl \Lambda_k =  \text{diag} \big( \tl \q(k, 1), \tl \q(k, 2), \ldots, \tl \q(k, d) \big), \quad \text{with} \ \ \tl \q(k, i) = \alpha_{\nu(k, i)} \ind\{i \in Y_k\} / \tl \alpha_k \in [0,1],$$
$$\hat M_{k+1} \= M_{k+1} / r(n) = F(x_k, \zeta_{k+1})/r(n) \ \, \text{(by Assum.~\ref{cond-alt-ns}),  \ and}  \ \  \hat \epsilon_{k+1} \= \epsilon_{k+1}/r(n).$$ 
 
Let us introduce another noise sequence $\{\hat M^o_k\}$ related to $\{\hat M_{k}\}$. For $n \geq 0$ and $k \geq 0$, let
\begin{equation} \label{eq-alt-prf1}
    \hat M^o_{k+1}  \= F (0, \zeta_{k+1})/r(n) \ \ \text{if} \ m(n) \leq k < m(n+1).
\end{equation}  
Equivalently, by the definition of $m(n)$, for each $k \geq 0$, 
\begin{equation} \notag
 \hat M^o_{k+1} = F (0, \zeta_{k+1}) / r(\ell(k)), \ \  \ \text{where} \  \ \ell(k) \= \max \{ \ell \geq 0: T_\ell \leq \tl t(k) \}.
\end{equation} 
Observe that $r(\ell(k)) = \| x_{m(\ell(k))} \| \vee 1$ is $\F_k$-measurable. 
Therefore, by Assum.~\ref{cond-alt-ns}(ii) and (\ref{eq-alt-prf0}),
\begin{equation} \label{eq-alt-prf2}
  \E [ \hat M^o_{k+1} \mid \F_k] = 0, \qquad \E [ \| \hat M^o_{k+1} \|^2 \mid \F_k ] \leq K_F, \quad \forall \, k \geq 0.
\end{equation}
Moreover, by the Lipschitz continuity property of $F$, for $m(n) \leq k < m(n+1)$,
\begin{equation} \label{eq-alt-prf3}
   \| \hat M_{k+1} - \hat M^o_{k+1} \| = \frac{\| F(x_k, \zeta_{k+1}) - F(0, \zeta_{k+1}) \|}{r(n)}  \leq  \frac{L_F \| x_k\|}{r(n)} = L_F \| \hat x^n(\tl t(k)) \|.
\end{equation} 
 
\begin{mylemma} \label{lem-alt1}
The sequence $\xi_n^o \= \sum_{k=0}^{n-1} \tl \alpha_k \tl \Lambda_k \hat M^o_{k+1}$ (with $\xi_0^o = 0$) converges a.s.\ in $\R^d$.
\end{mylemma}

\begin{proof}
Since, for all $k \geq 0$, the stepsizes $\tl \alpha_k$ and the entries of $\tl \Lambda_k$ are bounded by deterministic constants, it follows from (\ref{eq-alt-prf2}) that $(\xi_n^o, \F_n)$ is a square-integrable martingale and moreover, 
$$\sum_{n=0}^\infty \E \left[ \| \xi^o_{n+1} - \xi^o_n \|^2 \mid \F_n \right] \leq 
\sum_{n=0}^\infty \tl \alpha^2_n \| \tl \Lambda_n \|^2 \, \E \left[ \| \hat M^o_{n+1} \|^2 \mid \F_n \right] 
\leq K_F \sum_{n=0}^\infty \tl \alpha_n^2 \| \tl \Lambda_n \|^2 < \infty, \ \ \ a.s.$$
(since $\sum_n \tl \alpha_n^2 < \infty$ a.s.).
Then by \cite[Prop.\ VII-2-3(c)]{Nev75}, almost surely, $\xi_n^o$ converges in $\R^d$.
\end{proof} 
 
\begin{mylemma} \label{lem-alt2}
$\sup_{n \geq 0} \sup_{t \in [T_n, T_{n+1}]} \| \hat x^n(t) \|  < \infty$ a.s.
\end{mylemma}

\begin{proof}
As in \cite[Lem.\ 4.5]{Bor23}, we will show that $\sup_{t \in [T_n, T_{n+1}]} \| \hat x^n(t) \|$ can be bounded by a number independent of $n$. For each $n \geq 0$, using (\ref{eq-alt-hx}), we have that for $k$ with $m(n) \leq k < m(n+1)$,
$$ \hat x^n(\tl t(k+1))  =   \hat x^n(\tl t(m(n)))  + \sum_{i=m(n)}^{k} \tl \alpha_i \tl \Lambda_i h_{r(n)}(\hat x^n(\tl t(i))) + 
\sum_{i=m(n)}^{k} \tl \alpha_i \tl \Lambda_i \hat M_{i+1} + \sum_{i=m(n)}^{k} \tl \alpha_i \tl \Lambda_i \hat \epsilon_{i+1}.$$
Similarly to the proof of \cite[Lem.\ 4.5]{Bor23}, we proceed to bound $\| \hat x^n(\tl t(k+1))\|$ by bounding the norm of each term on the r.h.s.\ of the above equation. By the definition of $\hat x(\cdot)$, we have $\| \hat x^n(\tl t(m(n)) \| \leq 1$. Using the Lipschitz continuity of $h_c$ (Assum.~\ref{cond-h}) and the fact that $\sup_{i \geq 0} \| \tl \Lambda_i\|\leq \tl C$ for some deterministic constant $\tl C$, we can bound the norm of the second term by $\sum_{i=m(n)}^{k} \tl \alpha_i \tl C ( \| h(0)\| + L_h \| \hat x^n(\tl t(i))\| )$. For the forth term, by Assum.~\ref{cond-ns}(ii), we have $\|\hat \epsilon_{i+1} \| = \|  \epsilon_{i+1} \|/r(n) \leq \delta_{i+1} (1 + \| \hat x^n(\tl t(i))\| )$, so $\left\| \sum_{i=m(n)}^{k} \tl \alpha_i \tl \Lambda_i \hat \epsilon_{i+1} \right\| \leq \sum_{i=m(n)}^{k} \tl \alpha_i \tl C B_\delta (1 + \| \hat x^n(\tl t(i))\|)$, where $B_\delta \= \sup_{i \geq 1} \delta_i < \infty$ a.s..  
For the third term, we use (\ref{eq-alt-prf3}) and Lem.~\ref{lem-alt1} to obtain
\begin{align*}
    \left\| \sum_{i=m(n)}^{k} \tl \alpha_i \tl \Lambda_i \hat M_{i+1}  \right\| & \leq 
     \left\| \sum_{i=m(n)}^{k} \tl \alpha_i \tl \Lambda_i \hat M^o_{i+1} \right\| +  
       \sum_{i=m(n)}^{k} \tl \alpha_i  \| \tl \Lambda_i \|  \left\| \hat M_{i+1} - \hat M^o_{i+1} \right\|  \\
       & \leq \| \xi^o_{k+1} - \xi^o_{m(n)} \| +  \sum_{i=m(n)}^{k} \tl \alpha_i  \tl C  \cdot L_F\| \hat x^n(\tl t(i)) \| \\
       & \leq 2 B + L_F \tl C \!\sum_{i=m(n)}^{k} \tl \alpha_i  \| \hat x^n(\tl t(i)) \|,
 \end{align*}  
 where $B \= \sup_i \| \xi^o_i\| < \infty$ a.s.\ (Lem.~\ref{lem-alt1}). Observe also that by the definitions of $m(n)$, $m(n+1)$, and $\tl \alpha_i$, we have
$\sum_{i=m(n)}^{m(n+1) -1} \tl \alpha_i < T + \tl \alpha_{m(n+1) -1} \leq T + d \bar \alpha$,  where $\bar \alpha \= \sup_j \alpha_j < \infty$.
By combining the preceding derivations, we obtain
 $$
 \| \hat x^n(\tl t(k+1)) \| \leq 1 +  2 B +  \tl C (T+d \bar \alpha) (\| h(0) \| + B_\delta) + \tl C(L_h + L_F + B_\delta) \!\sum_{i=m(n)}^{k} \tl \alpha_i \| \hat x^n(\tl t(i))\|.
 $$
 Then by the discrete Gronwall inequality \cite[Lem.\ B.2]{Bor23}, for all $k$ with $m(n) \leq k < m(n+1)$,
 $$  \| \hat x^n(\tl t(k+1)) \| \leq \left(1 +  2 B +  \tl C (T+d \bar \alpha) (\| h(0) \| + B_\delta) \right) e^{\tl C (L_h+ L_F + B_\delta) (T+d \bar \alpha)}.$$
This shows that almost surely, $\sup_{t \in [T_n, T_{n+1}]} \| \hat x(t) \|$ can be bounded by a finite (random) number independent of $n$, and therefore, $\sup_{n \geq 0} \sup_{t \in [T_n, T_{n+1}]} \| \hat x(t) \| < \infty$ a.s.
\end{proof}
 
With Lem.~\ref{lem-alt2}, we have established the boundedness of the scaled trajectory $\hat x(\cdot)$. This has the following implication, which will be needed shortly in relating $\{\hat x^n(\cdot)\}$ to ODE solutions:

\begin{mylemma} \label{lem-alt3}
Almost surely, as $n \to \infty$, $\hat \epsilon_n \to 0$, and $\xi_n \= \sum_{k=0}^{n-1} \tl \alpha_k \tl \Lambda_k \hat M_{k+1}$ converges in $\R^d$.
\end{mylemma}

\begin{proof}
By the definition of $\{\hat \epsilon_k\}$ and Assum.~\ref{cond-ns}(ii), we have that for $m(n) \leq k < m(n+1)$, $\|\hat \epsilon_{k+1} \| = \|  \epsilon_{k+1} \|/r(n) \leq \delta_{k+1} (1 + \| \hat x^n(\tl t(k))\|)$, where $\delta_{k} \to 0$ a.s., as $k \to \infty$. 
By Lem.~\ref{lem-alt2}, this implies $\hat \epsilon_k \to 0$ a.s., as $k \to \infty$.

The proof of the a.s.\ convergence of $\{\xi_n\}$ is similar to the proof of Lem.~\ref{lem-cvg-3} in Sec.~\ref{sec-cvg}. Specifically, for integers $N \geq 1$, we define stopping times $\tau_N$ and auxiliary variables ${\hat M}^{(N)}_k$ by
\begin{align*}
  \tau_N  & \= \, \min \left\{ k \geq 0 \, \big|\,  \| \hat x^n(\tl t(k))\| > N, \, m(n) \leq k < m(n+1), \, n \geq 0 \right\},  \\
 {\hat M}^{(N)}_{k+1} & \= \, \ind \{ k < \tau_N \} \hat M_{k+1},  \quad k \geq 0.
\end{align*} 
Using Assum.~\ref{cond-alt-ns}, \eqref{eq-alt-prf0b}, and the definition of $\hat M_{k+1}$, we have that for each $N$,  $\{{\hat M}^{(N)}_k \}_{k \geq 1}$ is a martingale difference sequence with 
$$\E[ \| {\hat M}^{(N)}_{k+1} \|^2 \!\mid\! \F_k ] \leq \ind \{ k < \tau_N \} \cdot  K_F  ( 1 +  \| \hat x^{n_k}(\tl t(k))\|^2) \leq  K_F (1 + N^2 ),$$
where $n_k$ is such that $m(n_k) \leq k < m(n_k+1)$ (more specifically, $n_k$ is given by $n_k \= \max \{ \ell \geq 0: T_\ell \leq \tl t(k) \}$ and thus $\F_k$-measurable).
As in the proof of Lem.~\ref{lem-cvg-3}, it then follows that the sequence $\{ \xi^{(N)}_n\}_{n \geq 0}$ given by $\xi^{(N)}_n \= \sum_{k=0}^{n-1} \tl \alpha_k \tl \Lambda_k {\hat M}^{(N)}_{k+1}$ with $\xi^{(N)}_0 \= 0$ is a square-integrable martingale and converges a.s.\ in $\R^{d}$ by \cite[Prop.\ VII-2-3(c)]{Nev75}. Since $\sup_{n \geq 0} \sup_{t \in [T_n, T_{n+1}]} \| \hat x^n(t) \|  < \infty$ a.s.\ by Lem.~\ref{lem-alt2}, the definitions of $\tau_N$ and $\{{\hat M}^{(N)}_k\}$ imply that almost surely, $\{\xi_n\}_{n \geq 1}$ coincides with $\{\xi^{(N)}_n\}_{n \geq 1}$ for some sample path-dependent value of $N$. This leads to the a.s.\ convergence of $\{\xi_n\}$ in $\R^{d}$.
\end{proof}

Using Lems.~\ref{lem-alt2} and \ref{lem-alt3}, we can follow the same proof steps of \cite[Lem.\ 2.1]{Bor23} to obtain
\begin{equation}
   \lim_{n \to \infty} \sup_{t \in [T_n, T_{n+1}]} \left\| \hat x^n(t) - x^n(t) \right\| = 0 \ \ \ a.s.,
\end{equation}   
where $x^n(\cdot)$ is redefined in this case to be the unique solution of the ODE associated with the scaled function $h_{r(n)}$ and the piecewise constant trajectory $\tl \lambda(\cdot) \in \tilde \Upsilon$ given by (\ref{eq-tlambda}) in Sec.~\ref{sec-prel-ana}:
$$ \dot{x}(t) = \tl \lambda(t) \, h_{r(n)} (x(t))  \ \ \  \text{with} \ x^n(T_n) = \hat x(T_n) = x_{m(n)}/r(n).$$
 
From this point forward, we can argue similarly to Sec.~\ref{sec-stab-prf2} to establish the a.s.\ boundedness of the iterates $\{x_n\}$ from algorithm (\ref{eq-alg0}). Since, in this case, as $t \to \infty$, $\tl \lambda(t + \cdot)$ converges in $\tl \Upsilon$ to the unique limit point $\bar \lambda(\cdot) \equiv \tfrac{1}{d} I$ (Lem.~\ref{lem-cvg-2}), there is no need to consider multiple limit points as in Sec.~\ref{sec-stab-prf2}. Consequently, the proof arguments involved are slightly simpler.
\vspace*{-0.1cm}

\section*{Acknowledgements}
\vspace*{-0.3cm}

We thank Prof.\ Eugene Feinberg for helpful discussion on average-reward SMDPs and Dr.\ Martha Steenstrup for critical feedback on parts of our earlier draft. 
We also thank an anonymous reviewer for valuable comments, particularly suggesting the potential for sharper convergence results, which motivated our study of shadowing properties and led us to strengthen the analysis.
In preparing this paper, we used OpenAI ChatGPT (GPT-4 and GPT-5) to refine the writing style.

\addcontentsline{toc}{section}{References} 
\bibliographystyle{apa} 
\let\oldbibliography\thebibliography
\renewcommand{\thebibliography}[1]{%
  \oldbibliography{#1}%
  \setlength{\itemsep}{0pt}%
}
{\fontsize{9}{11} \selectfont
\bibliography{asyn_sa_arXiv3a.bib}}

\end{document}

%% file: macros-asyn-sa-arxiv-v3.tex
\usepackage{amsmath,amsthm,amssymb,mathrsfs} 
\usepackage[titletoc,title]{appendix}
\usepackage[numbers,sort&compress]{natbib}
\usepackage[hidelinks]{hyperref}
\usepackage{cleveref}

\usepackage{chngcntr}
\usepackage{titlesec}

\usepackage{lipsum}
\usepackage[hang,flushmargin]{footmisc}

\allowdisplaybreaks

\newcommand\blfootnote[1]{%
  \begingroup
  \renewcommand\thefootnote{}\footnote{#1}%
  \addtocounter{footnote}{-1}%
  \endgroup
}

\newtheorem{theorem}{Theorem}[section]
\newtheorem{assumption}{Assumption}[section]
\newtheorem{lemma}{Lemma}[section]
\newtheorem{prop}{Proposition}[section]
\newtheorem{cor}{Corollary}[section]

\newtheorem{remark}{Remark}[section]

\newtheorem{myassumption}{Assumption}[section] 
\newtheorem{mylemma}{Lemma}[section]

\numberwithin{equation}{section}

\let\secmark\S
\usepackage{enumitem}
\setlist[itemize]{itemsep=0pt,parsep=2pt,topsep=2pt}
\setlist[enumerate]{itemsep=0pt,parsep=2pt,topsep=2pt}
\def\E{\mathbb{E}}
\def\R{\mathbb{R}}
\def\I{\mathcal{I}}
\def\F{\mathcal{F}}
\def\C{\mathcal{C}}
\def\Z{\mathcal{Z}}
\def\q{b}
\def\tl{\tilde}
\DeclareMathAlphabet{\mathbbb}{U}{bbold}{m}{n}
\newcommand{\ind}{\mathbbb{1}}
\def\={\mathop{:=}}

\def\1{\mathbf{1}}

\def\nmid{\,|\,}
\def\myqed{\qed}

\def\unitS{\mathbb{S}_1}
\def\S{\mathcal{S}}

\def\K{\mathcal{K}}

%% file: asyn_sa_arXiv_v3.bbl
\begin{thebibliography}{}

\bibitem[\protect\astroncite{Abounadi et~al.}{2001}]{ABB01}
Abounadi, J., Bertsekas, D.~P., and Borkar, V.~S. (2001).
\newblock Learning algorithms for {Markov} decision processes with average
  cost.
\newblock {\em SIAM J. Control Optim.}, 40(3):681--698.

\bibitem[\protect\astroncite{Abounadi et~al.}{2002}]{ABB02}
Abounadi, J., Bertsekas, D.~P., and Borkar, V.~S. (2002).
\newblock Stochastic approximation for nonexpansive maps: applications to
  {Q-learning} algorithms.
\newblock {\em SIAM J. Control Optim.}, 41(1):1--22.

\bibitem[\protect\astroncite{Bena\"{i}m}{1996}]{Ben96}
Bena\"{i}m, M. (1996).
\newblock A dynamical systems approach to stochastic approximations.
\newblock {\em SIAM J. Control Optim.}, 34(2):437--472.

\bibitem[\protect\astroncite{Bena\"{i}m}{1999}]{Ben99}
Bena\"{i}m, M. (1999).
\newblock Dynamics of stochastic approximation algorithms.
\newblock In {\em S\'{e}minaire de Probabilities}, volume 1769 of {\em Lecture
  Notes in Mathematics}, pages 1--69. Springer, New York.

\bibitem[\protect\astroncite{Bena\"{i}m and Hirsch}{1996}]{BeH96}
Bena\"{i}m, M. and Hirsch, M.~W. (1996).
\newblock Asymptotic pseudotrajectories and chain recurrent flows, with
  applications.
\newblock {\em J. Dyn. Differ. Equ.}, 8(1):141--176.

\bibitem[\protect\astroncite{Bertsekas and Tsitsiklis}{1996}]{BeT96}
Bertsekas, D.~P. and Tsitsiklis, J.~N. (1996).
\newblock {\em Neural-Dynamic Programming}.
\newblock Athena Scientific, Belmont.

\bibitem[\protect\astroncite{Bhatia and Szeg\"{o}}{2002}]{BhS02}
Bhatia, N.~P. and Szeg\"{o}, G.~P. (2002).
\newblock {\em Stability Theory of Dynamical Systems}.
\newblock Springer, Berlin.

\bibitem[\protect\astroncite{Bhatnagar}{2011}]{Bha11}
Bhatnagar, S. (2011).
\newblock The {Borkar--Meyn} theorem for asynchronous stochastic
  approximations.
\newblock {\em Systems Control Lett.}, 60:472--478.

\bibitem[\protect\astroncite{Borkar}{1998}]{Bor98}
Borkar, V.~S. (1998).
\newblock Asynchronous stochastic approximations.
\newblock {\em SIAM J. Control Optim.}, 36(3):840--851.

\bibitem[\protect\astroncite{Borkar}{2000}]{Bor00}
Borkar, V.~S. (2000).
\newblock Erratum: Asynchronous stochastic approximations.
\newblock {\em SIAM J. Control Optim.}, 38(2):662--663.

\bibitem[\protect\astroncite{Borkar}{2023}]{Bor23}
Borkar, V.~S. (2023).
\newblock {\em Stochastic Approximations: A Dynamical Systems Viewpoint}.
\newblock Springer and Hindustan Book Agency, Singapore and New Delhi, 2nd
  edition.

\bibitem[\protect\astroncite{Borkar and Meyn}{2000}]{BoM00}
Borkar, V.~S. and Meyn, S. (2000).
\newblock The o.d.e. method for convergence of stochastic approximation and
  reinforcement learning.
\newblock {\em SIAM J. Control Optim.}, 38(2):447--469.

\bibitem[\protect\astroncite{Borkar and Soumyanath}{1997}]{BoS97}
Borkar, V.~S. and Soumyanath, K. (1997).
\newblock A new analog parallel scheme for fixed point computation, {Part I}:
  Theory.
\newblock {\em IEEE Trans. Circuits Systems---I Fund. Theory Appl.},
  44(4):351--355.

\bibitem[\protect\astroncite{Bowen}{1975}]{Bow75}
Bowen, R. (1975).
\newblock {${\omega}$}-limit sets of {Axiom A} diffeomorphisms.
\newblock {\em J. Differ. Equ.}, 18:333--339.

\bibitem[\protect\astroncite{Conley}{1978}]{Con78}
Conley, C.~C. (1978).
\newblock {\em Isolated Invariant Sets and the {Morse} Index}, volume~38 of
  {\em CBMS Regional Conference Series in Mathematics}.
\newblock AMS, Providence.

\bibitem[\protect\astroncite{Doob}{1953}]{Doo53}
Doob, J. (1953).
\newblock {\em Stochastic Processes}.
\newblock Wiley and Sons, New York.

\bibitem[\protect\astroncite{Dudley}{2002}]{Dud02}
Dudley, R.~M. (2002).
\newblock {\em Real Analysis and Probability}.
\newblock Cambridge University Press, Cambridge.

\bibitem[\protect\astroncite{Hall and Heyde}{1980}]{HaH80}
Hall, P. and Heyde, C.~C. (1980).
\newblock {\em Martingale Limit Theory and Its Application}.
\newblock Academic Press, New York.

\bibitem[\protect\astroncite{Hirsch}{1994}]{Hir94}
Hirsch, M.~W. (1994).
\newblock Asymptotic phase, shadowing and reaction-diffusion systems.
\newblock In {\em Differential Equations, Dynamical Systems and Control
  Science}, volume 152 of {\em Lecture Notes in Pure and Applied Mathematics},
  pages 87--99. Marcel Dekker, New York.

\bibitem[\protect\astroncite{Kushner and Yin}{2003}]{KuY03}
Kushner, H.~J. and Yin, G.~G. (2003).
\newblock {\em Stochastic Approximation and Recursive Algorithms and
  Applications}.
\newblock Springer, New York, 2nd edition.

\bibitem[\protect\astroncite{Neveu}{1975}]{Nev75}
Neveu, J. (1975).
\newblock {\em Discrete Parameter Martingales}.
\newblock North-Holland, Amsterdam.

\bibitem[\protect\astroncite{Puterman}{1994}]{Put94}
Puterman, M.~L. (1994).
\newblock {\em Markov Decision Processes: Discrete Stochastic Dynamic
  Programming}.
\newblock John Wiley \& Sons.

\bibitem[\protect\astroncite{Ramaswamy et~al.}{2021}]{RBQ20}
Ramaswamy, A., Bhatnagar, S., and Quevedo, D.~E. (2021).
\newblock Asynchronous stochastic approximations with asymptotically biased
  errors and deep multiagent learning.
\newblock {\em IEEE Trans. Automat. Contr.}, 66(9):3969--3983.

\bibitem[\protect\astroncite{Schweitzer}{1971}]{Sch71}
Schweitzer, P.~J. (1971).
\newblock Iterative solution of the functional equations of undiscounted
  {Markov} renewal programming.
\newblock {\em J. Math. Anal. Appl.}, 34(3):495--501.

\bibitem[\protect\astroncite{Schweitzer and Federgruen}{1977}]{ScF77}
Schweitzer, P.~J. and Federgruen, A. (1977).
\newblock The asymptotic behavior of undiscounted value iteration in {Markov}
  decision problems.
\newblock {\em Math. Oper. Res.}, 2(4):360--381.

\bibitem[\protect\astroncite{Schweitzer and Federgruen}{1978}]{ScF78}
Schweitzer, P.~J. and Federgruen, A. (1978).
\newblock The functional equations of undiscounted {Markov} renewal
  programming.
\newblock {\em Math. Oper. Res.}, 3(4):308--321.

\bibitem[\protect\astroncite{Tsitsiklis}{1994}]{Tsi94}
Tsitsiklis, J. (1994).
\newblock Asynchronous stochastic approximation and {Q-learning}.
\newblock {\em Mach. Learn.}, 16:195--202.

\bibitem[\protect\astroncite{Walter}{1998}]{Wal98}
Walter, W. (1998).
\newblock {\em Ordinary Differential Equations}.
\newblock Springer, New York.

\bibitem[\protect\astroncite{Wan et~al.}{2021a}]{WNS21b}
Wan, Y., Naik, A., and Sutton, R.~S. (2021a).
\newblock Average-reward learning and planning with options.
\newblock In {\em Proc.\ NeurIPS}, pages 22758--22769.

\bibitem[\protect\astroncite{Wan et~al.}{2021b}]{WNS21a}
Wan, Y., Naik, A., and Sutton, R.~S. (2021b).
\newblock Learning and planning in average-reward {Markov} decision processes.
\newblock In {\em Proc.\ ICML}, pages 10653--10662.

\bibitem[\protect\astroncite{Wan et~al.}{2024}]{WYS24}
Wan, Y., Yu, H., and Sutton, R.~S. (2024).
\newblock On convergence of average-reward {Q-learning} in weakly communicating
  {Markov} decision processes.
\newblock \href{https://arxiv.org/abs/2408.16262}{arXiv:2408.16262}.

\bibitem[\protect\astroncite{White}{1963}]{Whi63}
White, D.~J. (1963).
\newblock Dynamic programming, {Markov} chains, and the method of successive
  approximations.
\newblock {\em J. Math. Anal. Appl.}, 6(3):373--376.

\bibitem[\protect\astroncite{Yu and Bertsekas}{2013}]{YuB13}
Yu, H. and Bertsekas, D.~P. (2013).
\newblock On boundedness of {Q-learning} iterates for stochastic shortest path
  problems.
\newblock {\em Math. Oper. Res.}, 38:209--227.

\bibitem[\protect\astroncite{Yu et~al.}{2025}]{YWS25}
Yu, H., Wan, Y., and Sutton, R.~S. (2025).
\newblock Average-reward reinforcement learning in semi-{Markov} decision
  processes via relative value iteration.
\newblock \href{https://arxiv.org/abs/2512.06218}{arXiv:2512.06218}.

\end{thebibliography}
